\documentclass[sigconf]{aamas} 

\usepackage{balance}
\usepackage{amsmath}
\usepackage{nicefrac}
\usepackage{xcolor}
\usepackage{graphicx}
\usepackage{caption}
\usepackage{subfig}
\usepackage{newtxtext}
\usepackage{xspace}
\usepackage{cleveref}
\usepackage{algorithm}
\usepackage[noend]{algorithmic}
\usepackage{amsthm}
\usepackage{multirow}
\usepackage{dblfloatfix}

\newcommand{\shivaram}[1]{\textcolor{red}{#1}}
\newcommand{\MM}{\ensuremath{\mathcal M}\xspace}

\newtheorem{theorem}{Theorem}

\newtheorem{proposition}[theorem]{Proposition}
\newtheorem{lemma}[theorem]{Lemma}

\newtheorem{definition}[theorem]{Definition}

\DeclareMathOperator*{\argmax}{argmax}
\DeclareMathOperator*{\argmin}{argmin}
\DeclareMathOperator*{\eqdef}{\stackrel{\text{\tiny def}}{=}}
\Crefname{algorithm}{Algorithm}{Algorithms}

\doi{PYUD1139}

\makeatletter
\gdef\@copyrightpermission{
  \begin{minipage}{0.2\columnwidth}
   \href{https://creativecommons.org/licenses/by/4.0/}{\includegraphics[width=0.90\textwidth]{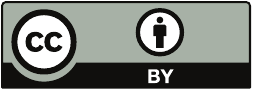}}
  \end{minipage}\hfill
  \begin{minipage}{0.8\columnwidth}
   \href{https://creativecommons.org/licenses/by/4.0/}{This work is licensed under a Creative Commons Attribution International 4.0 License.}
  \end{minipage}
  \vspace{5pt}
}
\makeatother

\setcopyright{ifaamas}
\acmConference[AAMAS '26]{Proc.\@ of the 25th International Conference
on Autonomous Agents and Multiagent Systems (AAMAS 2026)}{May 25 -- 29, 2026}
{Paphos, Cyprus}{C.~Amato, L.~Dennis, V.~Mascardi, J.~Thangarajah (eds.)}
\copyrightyear{2026}
\acmYear{2026}
\acmDOI{}
\acmPrice{}
\acmISBN{}

\acmSubmissionID{1043}

\title{On-line Learning in Tree MDPs by Treating~Policies as\\Bandit Arms}

\author{Anvay Shah}
\affiliation{
  \institution{Indian Institute of Technology Bombay}
  \city{Mumbai}
  \country{India}}
\email{anvay@cse.iitb.ac.in}

\author{Ramsundar Anandanarayanan}
\affiliation{
  \institution{Indian Institute of Technology Bombay}
  \city{Mumbai}
  \country{India}}
\email{ramsundar@cse.iitb.ac.in}

\author{Sharayu Moharir}
\affiliation{
  \institution{Indian Institute of Technology Bombay}
  \city{Mumbai}
  \country{India}}
\email{sharayum@ee.iitb.ac.in}

\author{Shivaram Kalyanakrishnan}
\affiliation{
  \institution{Indian Institute of Technology Bombay}
  \city{Mumbai}
  \country{India}}
\email{shivaram@cse.iitb.ac.in}

\begin{abstract}
A Tree Markov Decision Problem (T-MDP) is a finite-horizon MDP with a starting state $s_{1}$, in which every state is reachable from $s_{1}$ through exactly one state-action trajectory. T-MDPs arise naturally as abstractions of decision making in sequential games with perfect recall, against stationary opponents. We consider the problem of on-line learning in T-MDPs, both in the PAC and the regret-minimisation regimes. We show that well-known bandit algorithms---\textsc{Lucb} and \textsc{Ucb}---can be applied on T-MDPs by treating each policy as an arm. The apparent technical challenge in this approach is that the number of policies is exponential in the number of states. Our main innovation is in the design of confidence bounds based on data shared by the policies, so that the bandit algorithms can yet be implemented with polynomial memory and per-step computation. We obtain instance-dependent upper bounds on sample complexity and regret that sum a ``gap term'' from every terminal state, rather than every policy. Empirically, our algorithms consistently outperform available alternatives on a suite of hidden-information games.
\end{abstract}


\newcommand{\BibTeX}{\rm B\kern-.05em{\sc i\kern-.025em b}\kern-.08em\TeX}

\begin{document}


\pagestyle{fancy}
\fancyhead{}


\maketitle 


\section{Introduction}
A chief contributor to the rapid advent of artificial intelligence (AI) in the last couple of decades is the widespread adoption of data-driven algorithms. In the realm of decision making, major practical successes---in areas such as robotics~\cite{DBLP:journals/ijrr/KoberBP13}, game-playing~\cite{DBLP:journals/nature/SilverSSAHGHBLB17, DBLP:journals/nature/WurmanBKMS0CDE022}, large language models~\cite{DBLP:conf/nips/Ouyang0JAWMZASR22, DBLP:journals/corr/abs-2501-12948}, financial trading~\cite{DBLP:conf/nips/MoodyS98}---have been achieved through reinforcement learning (RL)~\cite{DBLP:books/lib/SuttonB2018}. The theoretical study of RL~\cite{DBLP:conf/colt/Fiechter94, DBLP:journals/ml/KearnsS02, DBLP:journals/ml/AzarMK13} has regularly benefitted its empirical progress.

In this paper, we consider on-line learning on sequential decision making problems that can be formalised as Tree Markov Decision Problems (T-MDPs). A T-MDP is a finite-horizon MDP in which the state-transitions form a \textit{tree}, rooted at a starting state $s_1$. Thus, every state in a T-MDP is reachable from $s_1$ through exactly one state-action trajectory. T-MDPs arise commonly as abstractions of extensive-form games with perfect recall~\cite{shoham2008multiagent}, against fixed opponents. The key challenge in these games is their large scale (amplified by hidden information) and the uncertainty regarding the opponent's strategy. Our work is motivated by such games, where the goal is to learn a {\em best response} against an {\em unknown} but stationary opponent, through repeated game interactions. This setup gives rise to an on-line learning problem in a T-MDP, in which the agent's states are action-observation histories.

A natural approach for designing learning algorithms for MDPs is to view them as generalisations of multi-armed bandits, their well-understood stateless counterparts~\cite{DBLP:journals/jmlr/Even-DarMM06}. Can on-line learning algorithms for bandits be appropriately generalised for MDPs? Simply put, our paper is the substantiation of an affirmative answer for the special case of T-MDPs. We take up two well-known bandit algorithms: (1) \textsc{Lucb} \cite{DBLP:conf/icml/KalyanakrishnanTAS12}, which achieves order-optimal sample complexity in the PAC setting, and (2) \textsc{Ucb} \cite{DBLP:journals/ml/AuerCF02}, whose regret is within a constant factor of optimal. We adapt these algorithms to T-MDPs, and denote the generalisations \textsc{Lucb-T} and \textsc{Ucb-T}.

The generalisations share a common core, in which each policy for the T-MDP is treated as a bandit arm. The main technical challenge is that the number of policies is \textit{exponential} in the size of the T-MDP. We propose a framework for policies to share data, tying it to a concentration inequality that we specifically establish for certain families of dependent random variables. We obtain instance-specific upper bounds on the sample complexity (for \textsc{Lucb-T}) and regret (for \textsc{Ucb-T}), both of which include a ``gap term'' for each terminal state in the T-MDP (whereas a na{\"i}ve implementation would yield such a term for every policy). The tree structure also facilitates efficient computation: these algorithms need only a polynomial number of operations to make each decision. 


Our algorithms assume that the reward function of the T-MDP is \textit{known}, and that the only parameters to estimate are transition probabilities. While restrictive in some cases, this assumption is reasonable for our motivating domain of imperfect information games. We implement \textsc{Lucb-T} and \textsc{Ucb-T} on three games with varying state-space sizes:  Kuhn Poker \cite{kuhnpoker} (10's of states), Leduc Poker \cite{BayesBluffOppModelling} (100's of states), and Reconnaissance Blind Tic-Tac-Toe (RBT) \cite{rbt} (millions of states). RBT is inspired by Reconnaissance Blind Chess \cite{pmlr-v220-gardner23a, pmlr-v176-perrotta22a}, and is significantly larger than test problems considered in the literature. We observe that our algorithms consistently outperform alternatives, especially as the problem size increases. Our experiments also throw light on gaps between theory and practice, especially in the PAC setting.
In summary, our contributions include
\begin{enumerate}
     \item a new analytical result (Section~\ref{sec:confidencebounds}) with potentially broader applicability to the analysis of learning in MDPs;
     \item conceptually-simple algorithms for T-MDPs (Section~\ref{sec:algorithmsandnalysis});
     \item strong theoretical guarantees in the form of instance-specific upper bounds on sample complexity (Theorem~\ref{thm:lucb-t-sc}) and regret (Theorem~\ref{thm:ucb-t-regret}); and 
     \item empirical validation on large T-MDPs, arising from well-known imperfect-information games (Section~\ref{sec:experiments}).
 \end{enumerate}
We also publish our code\footnote{Codebase: \url{https://github.com/anvay09/On-line-Learning-in-Tree-MDPs}.}, so our 
work can be reproduced and built upon. We begin by formalising our problem statements.

\section{Problem Description}

In this section, we define Tree MDPs, and thereafter specify the PAC and regret-minimisation problems.

\subsection{Tree Markov Decision Problems}

\begin{figure}[b]
    \centering
\vspace{-0.4cm}
\includegraphics[width=0.7\linewidth]{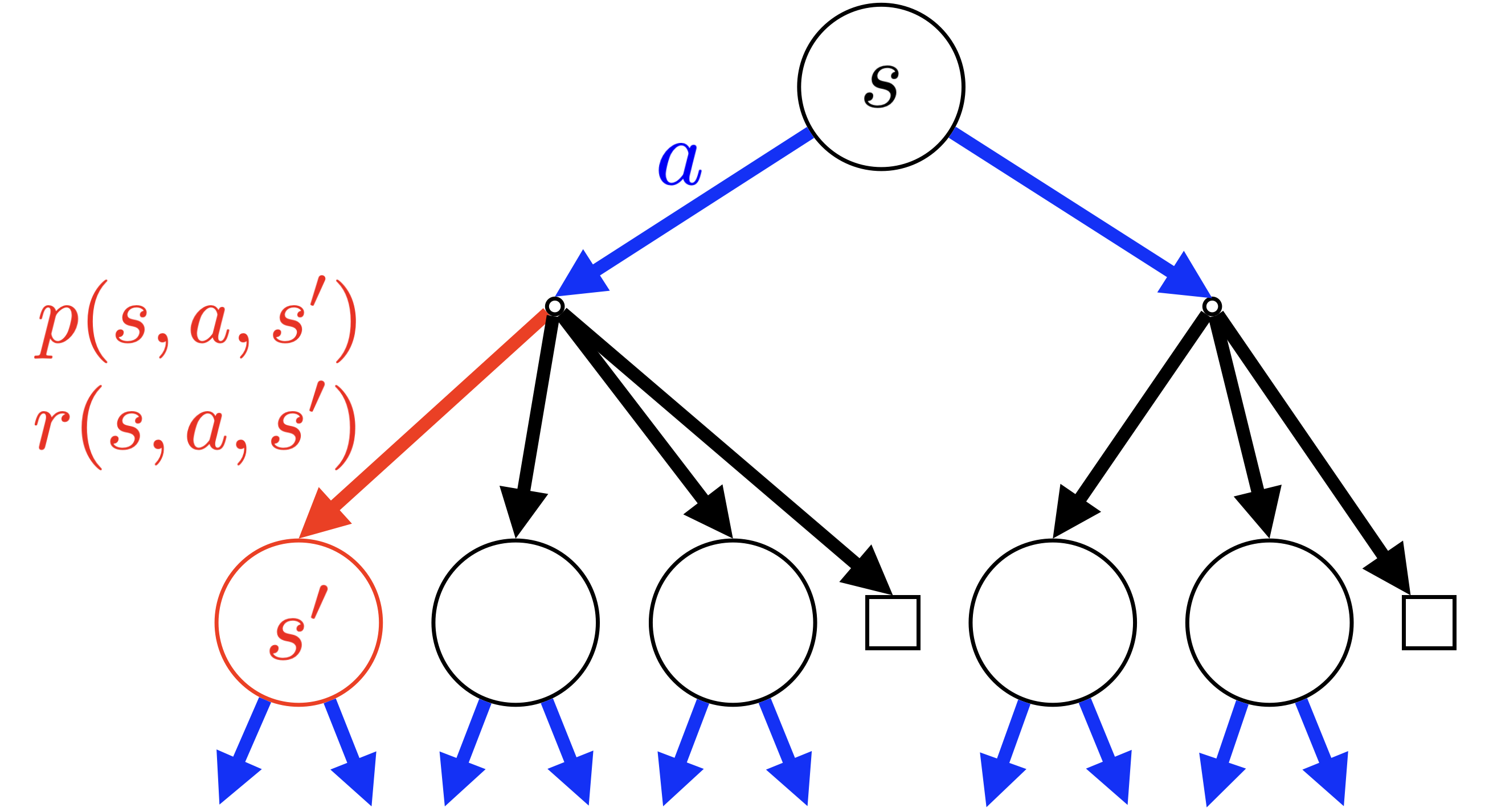}
 \vspace{-0.2cm}
    \caption{Transition from non-terminal state $s$ to state $s'$ upon taking action $a$. State $s^{\prime}$ could be non-terminal or terminal.}
    \label{fig:state}
    \Description{This image shows an example Tree MDP state.}
\end{figure}

A Tree MDP (T-MDP) $\mathcal{M}$ is specified by a 7-tuple $\langle \mathcal{S},\Sigma,\mathcal{A},p,r,\gamma, \mathcal{H} \rangle$. Here $\mathcal{S}$ and $\Sigma$ are sets of non-terminal and terminal states, respectively. We assume that $\mathcal{S}$ contains a starting state $s_{1}$. $\mathcal{A}$ is a set of actions, and $\mathcal{H} \geq 1$ is the horizon. Transition probabilities are given by 
the function $p$,
while the function $r$
specifies rewards. $\gamma \in [0, 1]$ is used to discount future rewards in the definition of values.


A T-MDP induces a tree rooted at $s_{1}$, which is the only state at ``level'' $1$. As illustrated in Figure~\ref{fig:state}, suppose $s \in \mathcal{S}$ is a state at level $1 \leq h \leq \mathcal{H}$. When the agent takes action $a$ from $s$, it goes to state $s^{\prime}$, which is either (1) terminal, or (2) a non-terminal state at level $h + 1$. The probability of this transition is $p(s, a, s^{\prime})$. It is a property of T-MDPs that each ``next state'' $s^{\prime}$ has exactly one parent $s$ from one lower level, from which $s^{\prime}$ is reachable through exactly one action $a$.

The transition from $s$ to $s^{\prime}$ by taking action $a$ merits a numeric reward $r(s, a, s^{\prime})$. We assume that this reward is known---an assumption that is reasonable for many practical applications \cite{DBLP:conf/nips/NgKJS03,DBLP:journals/nature/WurmanBKMS0CDE022}, particularly those from our motivating domain of games \cite{SHEPPARD2002241, DBLP:journals/corr/MnihKSGAWR13,  DBLP:conf/nips/TianGJ20}. This requirement is not uncommon in the on-line learning literature \cite{DBLP:journals/tcs/LattimoreH14a}, although many approaches can also accommodate stochastic and unknown rewards \cite{DBLP:conf/colt/WagenmakerSJ22, DBLP:conf/alt/TirinzoniMK23}. We assume that the discounted cumulative reward (that is, the \textit{return}) along each trajectory is bounded, and for convenient exposition taken to lie in $[0, 1]$.


A deterministic policy 
$\pi: \mathcal{S} \to \mathcal{A}$ specifies an action $\pi(s)$ for every non-terminal state $s$. Let $\Pi$ be the set of all deterministic policies for MDP $\mathcal{M}$. For every non-terminal state $s \in \mathcal{S}$, the \textit{value} under $\pi$ is given by the Bellman equation
\begin{align}
V^{\pi}(s) = \sum_{s^{\prime} \in \mathcal{S} \cup \Sigma} p(s, \pi(s), s^{\prime}) \left( r(s, \pi(s), s^{\prime}) + \gamma V^{\pi}(s^{\prime})\right),
\label{eqn:bellman}
\end{align}
with the convention that for every terminal state $\sigma \in \Sigma$, $V^{\pi}(\sigma) = 0$. The tree structure of $\mathcal{M}$ enables $\pi$ to be conveniently evaluated bottom-up. A policy $\pi$
that maximises the RHS in \eqref{eqn:bellman} for each state $s \in \mathcal{S}$ is an \textit{optimal} policy. Let $\Pi_{\text{opt}} \subseteq \Pi$ to be the set of optimal policies. We arbitrarily pick one optimal policy and denote it  $\pi^{\star}$.
Of particular interest to us is the value of policies at the start state $s_{1}$. For $\pi \in \Pi$, we define $V(\pi) \eqdef V^{\pi}(s_{1})$.

\subsection{On-line Learning Problems}

Our agent faces the challenge of not knowing the transition function $p$. However, the agent can interact with the MDP by playing complete episodes starting at $s_{1}$. Suppose that on episode $t \geq 1$, the agent plays a policy $\pi^{t}$. From each state $s$ that is reached, an action is taken according to $\pi^{t}$. The environment generates next state $s^{\prime}$ by sampling $p(s, \pi^{t}(s))$. If $s^{\prime}$ is terminal, we proceed to the next episode $t + 1$, wherein the agent can select a fresh policy $\pi^{t + 1}$. Note that the agent does \textit{not} have arbitrary access to sample any state-action pair in the tree, as is assumed in some other work~\cite{DBLP:journals/ml/AzarMK13,DBLP:conf/nips/ZanetteKB19}. Thus, if the agent wishes to explore some particular state, it may have to try for multiple episodes until randomness takes it down that state's path.

For each episode $t \geq 1$, a learning algorithm $\mathcal{L}$ must pick a policy $\pi^{t}$, based on the trajectories observed in the preceding $t - 1$ episodes. Recall that each trajectory is a state-action sequence beginning with $s_{1}$, and ending in a terminal state. Additionally, in the PAC setting, the learning algorithm may stop after any number of episodes $t$ and return $\pi^{t}$ as answer.

\subsubsection{PAC} In the PAC setting, a tolerance $\epsilon \in (0, 1)$ and a mistake probability $\delta \in (0, 1)$ are given as input to learning algorithm $\mathcal{L}$. The algorithm is required to stop with probability $1$ on every input T-MDP. Also, the policy it returns must satisfy $V(\pi) \geq V(\pi^{\star}) - \epsilon$ with probability at least $1 - \delta$. The main question is how many samples (that is, episodes) are required to provide such a guarantee. The PAC problem described above is one of ``pure exploration'', in the sense that the rewards accrued while learning are not of consequence.

\subsubsection{Regret.} The regret of an algorithm after $T$ episodes is $$R_{T} \eqdef \sum_{t = 1}^{T} \left( V(\pi^{\star}) - \mathbb{E}[V (\pi^{t})] \right).$$ We desire an algorithm that minimises regret on every T-MDP instance. Unlike the PAC setting, it is well-known that regret-minimisation needs to balance exploring (playing less-sampled policies) and exploiting (playing empirically-dominant policies).\\

From the specifications above, it is apparent that bandit algorithms such as \textsc{Lucb} and \textsc{Ucb} can be run ``as is'' on T-MDPs, by treating each policy as an arm. However, since the number of policies is $|\mathcal{A}|^{|\mathcal{S}|}$, the computational and memory requirements of na{\"i}ve implementations would be prohibitive on all but toy tasks. Our main contribution is an efficient, scalable implementation that exploits 
the structure of T-MDPs---and performs well both in theory and in practice. These algorithms scale polynomially in the number of possible trajectories from the starting state, which, in T-MDPs, is the same as the number of terminal states. In regular MDPs, the number of trajectories could, in general, be exponential in the number of states.

\section{Confidence Bounds on $V(\pi)$}
\label{sec:confidencebounds}

The main tool we devise is a confidence bound on $V(\pi)$ for arbitrary policy $\pi \in \Pi$. Facilitating this bound is a representation of $V(\pi)$ as a convex combination of terminal returns.

\subsection{Bottom-up View of $V(\pi)$}

The Bellman equations in \eqref{eqn:bellman} recursively define a parent node's value in terms of its children's. We introduce notation to unroll this recursion, up to the base case involving only terminal states.

Due to the tree property of T-MDPs, every state $s \in \mathcal{S} \cup \Sigma$ has a unique path from the root $s_{1}$; suppose for $s$ the path is $s_{1}, a_{1}, s_{2}, a_{2}, \dots, s_{m}, a_{m}, s$ for some $m \geq 1$. We denote by $\rho(s)$ the discounted cumulative reward (that is, the return) along this path, and by $q(s)$ the probability of reaching $s$ from $s_{1}$ if taking action $a_{i}$ from $s_{i}$ for $1 \leq i \leq m$. That is:
\begin{align}
\rho(s) \eqdef \sum_{i = 1}^{m} \gamma^{i - 1} r(s_{i}, a_{i}, s_{i + 1});\phantom{aaaa} 
q(s) \eqdef \prod_{i = 1}^{m} p(s_{i}, a_{i}, s_{i + 1}), 
\label{eqn:qrho}
\end{align}
where $s_{m + 1}$ is taken to be $s$. 

A policy $\pi \in \Pi$ is defined to be \textit{consistent} with terminal state $\sigma \in \Sigma$ if $\pi$ takes the action following each state in the path to $\sigma$.
Concretely, suppose $\sigma$ has path $(s_{1}, a_{1}, s_{2}, a_{2}, \dots, s_{m}, a_{m}, \sigma)$. Then $\pi \in \Pi$ is consistent with $\sigma$ if and only if for $1 \leq i \leq m$, $\pi(s_{i}) = a_{i}$. Observe that in general, a policy $\pi$ can be consistent with multiple terminal states; on any episode one of these states will be reached. On the other hand, also note that multiple policies can be consistent with a given terminal state $\sigma$---such policies would differ on states that are not on the path to $\sigma$. For $\pi \in \Pi$, let $X(\pi)$ denote the subset of terminal states with which $\pi$ is consistent; similarly, for $\sigma \in \Sigma$, let $Y(\sigma)$ denote the set of policies that are consistent with $\sigma$:
\begin{align}
X(\pi) = \{\sigma \in \Sigma: \pi \text{ is consistent with  } \sigma \},    
\label{eqn:X-def}
\\
Y(\sigma) = \{\pi \in \Pi: \pi \text{ is consistent with  } \sigma  \}.    
\label{eqn:Y-def}
\end{align}
Repeated expansion of the RHS of \eqref{eqn:bellman}, while invoking definitions from \eqref{eqn:qrho}, yields the following decomposition of $V(\pi)$; a detailed working is provided in Appendix~\ref{app:proof-v-terminal}.\footnote{Appendices are included in a longer version of the paper linked
from SK’s home page: \url{https://www.cse.iitb.ac.in/~shivaram/}.}
\begin{proposition}[Value as a convex combination of terminal returns]
\label{prop:v-terminal}
For $\pi \in \Pi$,    
\begin{align}
V^{\pi}(s_{1}) = \sum_{\sigma \in X(\pi)} q(\sigma) \cdot \rho(\sigma).   
\end{align}
\end{proposition}
The proposition tells us that in order to learn $V(\pi)$, it suffices to estimate $q(\sigma)$ for all $\sigma \in X(\pi)$. This reduces the problem of estimating values for all policies in $\Pi$ (an exponentially-sized set) to estimating the ``$q$''-s for all the elements of $\Sigma$. Moreover, we can get data for estimating $q(\sigma)$ merely by playing any policy that is consistent with $\sigma$---which may or may not reach $\sigma$ on any given episode. On the other hand, if $\rho(\sigma)$ was also unknown and had to be estimated, note that we would get sufficient information for it only from episodes that do reach $\sigma$. Thus, our assumption of known rewards carries a non-trivial advantage.

\subsection{Upper and Lower Confidence Bounds}

Consider the run of any arbitrary algorithm. On the $t$-th episode, $t \geq 1$, for each terminal state $\sigma \in \Sigma$, let $n^{t}(\sigma)$ denote the number of episodes so far on which the policy played was consistent with $\sigma$, and let $n^{t}_{+}(\sigma)$ denote the number of times $\sigma$ was reached. These counts can be maintained using $\Theta(|\Sigma|$) integer operations per episode. If policy $\pi$ is played on the $t$-th episode, and reaches $\sigma$, the updates are 
\begin{align}
n^{t + 1}(\sigma^{\prime}) &\gets     n^{t}(\sigma^{\prime}) + 1, \text{ for all } \sigma^{\prime} \in X(\pi); \label{eq:n-update}\\ n^{t + 1}_{+}(\sigma) &\gets    n^{t}_{+}(\sigma) + 1, \label{eq:nplus-update}
\end{align}
while for other terminal states the counts do not change from episode $t$ to episode $t + 1$. With a slight abuse of notation, we define
\begin{align*}
n^{t}(\pi) \eqdef \min_{\sigma \in X(\pi)} n^{t}(\sigma)
\end{align*}
for $\pi \in \Pi$ as a ``play count'' signifying the amount of usable data we have for evaluating $\pi$ after $t - 1$ episodes. Note that $\pi$ need not be played at all for $n^{t}(\pi)$ to be positive---we only need policies consistent with states in $X(\pi)$ to have been played.

Notice that the ratio $\widehat{q}^{t}(\sigma) \eqdef \frac{n^{t}_{+}(\sigma)}{n^{t}(\sigma)}$ is an unbiased estimator of $q(\sigma)$; consequently the empirical value estimate 

\begin{align*}
\widehat{V}^{t}(\pi) &\eqdef \sum_{\sigma \in X(\pi)} \widehat{q}^{t}(\sigma) \cdot \rho(\sigma)   \end{align*}
is an unbiased estimator of $V(\pi)$. Due to our convention that $\rho(\cdot)$ lies in $[0, 1]$, and since $\widehat{q}^{t}(\cdot)$ must also lie in $[0, 1]$, it follows that $\widehat{V}^{t}(\pi)$ must lie in $[0, |X(\pi)|]$. As seen shortly, we clip this estimate to $[0, 1]$ for ``one direction'' of our algorithm and analysis.

At the core of our technical contribution are the following confidence bounds that we propose on $V(\pi)$. For $\delta \in (0, 1)$,
\begin{align}
\text{ucb}^{t}(\pi, \delta) &\eqdef \widehat{V}^{t}(\pi) + \beta(n^{t}(\pi), \delta),\label{eq:ucb}\\
\text{lcb}^{t}(\pi, \delta) &\eqdef \min\left\{\widehat{V}^{t}(\pi), 1\right\} - \beta(n^{t}(\pi), \delta),  \label{eq:lcb}\\
\text{where } \beta(m, \delta) &\eqdef \sqrt{\frac{8}{3m} \ln \frac{1}{\delta}} \text{ for } m \geq 1.\nonumber
\end{align}

In \eqref{eq:ucb} and \eqref{eq:lcb}, observe that the ``confidence width'' $\beta(\cdot, \cdot)$ for $\pi$ depends on $n^{t}(\pi)$, which, in turn is determined by the \textit{least-played} terminal state of $\pi$. When $\pi$ is played, every terminal state in $X(\pi)$ is played; hence the confidence width of $\pi$ necessarily decreases. 

\begin{theorem}[Confidence bounds on policy's value]
\label{thm:confidencebounds}
Consider any policy $\pi \in \Pi$. Consider any given $t \geq 1$ and sequence $\left(n^{t}(\sigma)\right)_{\sigma \in X(\pi)}$ with $1 \leq n^{t}(\sigma) \leq t$ for all $\sigma \in X(\pi)$. During a run of any algorithm, suppose $t$ equals the number of episodes and 
$n^{t}(\sigma)$ equals the number of plays of $\sigma$ for each $\sigma \in X(\pi)$, the terminal states, respectively. Then, for $\delta \in (0, 1)$: $$\mathbb{P}\{V(\pi) \geq \text{ucb}^{t}(\pi, \delta)\} \leq \delta;~~ \mathbb{P} \{V(\pi) \leq \text{lcb}^{t}(\pi, \delta)\} \leq \delta.$$
\end{theorem}
The non-trivial proof of this theorem is a central contribution of our paper, and takes up the remainder of this section. The main technical challenge is the \textit{dependence} among the data used to compute $\widehat{V}^{t}(\pi)$. Applying results from the literature on concentration inequalities~\cite{doi:10.1137/panconesi93,garbe2018concentrationlipschitzfunctionsnegatively}, we establish that the particular nature  of this dependence yet permits the use of Chernoff bounds, which we then apply. Our exposition below assumes that $\pi$, $t$, and $\left(n^{t}(\sigma)\right)_{\sigma \in X(\pi)}$ are given.

\subsubsection{$\widehat{V}^{t}(\pi)$ as a weighted sum of Bernoullis}

For each $\sigma \in X(\pi)$ and $1 \leq i \leq n^{t}(\sigma)$, let the Bernoulli random variable $B(\sigma, i)$ denote the outcome of the $i$-th time that a policy consistent with $\sigma$ was played. $B(\sigma, i)$ is $1$ if this play reaches $\sigma$, otherwise it is $0$. We informally refer to these random variables as ``$B$-variables'', and denote their collection
$\mathcal{B}$. Note that $|\mathcal{B}| = \sum_{\sigma \in X(\pi)} n^{t}(\sigma)$. We observe that $\widehat{V}^{t}(\pi)$ is a weighted combination of $B$-variables, with non-negative weights.

\begin{align}
\widehat{V}^{t}(\pi) &= \sum_{\sigma \in X(\pi)} \frac{\rho(\sigma)}{n^{t}(\sigma)} \sum_{i = 1}^{n^{t}(\sigma)} B(\sigma, i).
\label{eq:Vhat}
\end{align}
Since $B(\sigma, i)$ has mean $q(\sigma)$; we observe from Proposition~\ref{prop:v-terminal} that $\mathbb{E}[\widehat{V}^{t}(\pi)]$ $= \sum_{\sigma \in X(\pi)} \rho(\sigma) \cdot q(\sigma) = V(\pi).$ For proving the theorem, we need to upper-bound the probability that $\widehat{V}^{t}(\pi)$ deviates from its expectation $V(\pi)$ by more than some amount in each direction.

\subsubsection{Negative cylinder dependence of $\mathcal{B}$} Notice that for $\sigma \in X(\pi)$ and $1 \leq i < j \leq n^{t}(\sigma)$, $B(\sigma, i)$ and $B(\sigma, j)$ are \textit{independent}, since they must  necessarily come from different episodes. However, For $\sigma, \sigma^{\prime} \in X(\pi)$ and $1 \leq i \leq j \leq n^{t}(\sigma)$, $B(\sigma, i)$ and $B(\sigma', j)$ \textit{need not} be independent, since the $i$-th play of $\sigma$ and the $j$-th play of $\sigma^{\prime}$  may be on the \textit{same} episode. At an intuitive level, it appears that this dependence between $B(\sigma, i)$ and $B(\sigma', j)$ should only help, since when one of them is $1$ (or $0$), the other has a higher probability of being $0$ (respectively $1$), thereby keeping the average more concentrated. Formally, these random variables are ``negative cylinder dependent'' (NCD).

\begin{definition}[Negative Cylinder Dependence~\cite{garbe2018concentrationlipschitzfunctionsnegatively}]
Bernoulli random variables $Z_{1}, Z_{2}, \dots, Z_{m}$, where $m \geq 2$, are negative cylinder dependent (NCD) if and only if for each $S \in \{1, 2, \dots, m\}$,
\begin{align*}
\mathbb{P}\left\{\cap_{i \in S} (Z_{i} = 1) \right\}
&\leq \prod_{i \in S} \mathbb{P}\{Z_{i} = 1\}; \text{ and}\\
\mathbb{P}\left\{\cap_{i \in S} (Z_{i} = 0) \right\}
&\leq \prod_{i \in S} \mathbb{P}\{Z_{i} = 0\}. 
\end{align*}
\end{definition}
\begin{lemma}[$B$-variables are NCD]
\label{lem:B-NCD}
The set of random variables $\mathcal{B}$ are Negative Cylinder Dependent.    
\end{lemma}

The intuition behind this claim is as follows. A subset of $B$-variables cannot all be $1$, when conditioned on the event that any two of them occur on the same episode. Nor can they all be $0$, when conditioned on the event that even one of the episodes generating them results in $0$ as the outcome for every terminal state in the complement of this subset. Our formal proof involves such conditioning of the LHS probabilities, removing terms from $0$-probability events, and re-aggregating the surviving terms. This working involves lengthy mathematical expansions, and for reasons of space is deferred to Appendix~\ref{app:proof-cb}. 

The standard recipe for analysing NCD variables is to introduce ``twins'', which are amenable to the application of Chernoff bounds.

\subsubsection{Twin variables} For each $B$-variable in $\mathcal{B}$, we define a corresponding ``twin-variable'' (called a ``$B_{0}$-variable'') that is also a Bernoulli with the same mean, but which is generated independently (of all the other variables). Thus, for each $\sigma \in X(\pi)$, $1 \leq i \leq n^{t}(\sigma)$, the Bernoulli random variable $B_{0}(\sigma, i)$ is generated independently, and satisfies $\mathbb{E}[B_{0}(\sigma, i)] = \mathbb{E}[B(\sigma, i)] = q(\sigma)$. We also define $\widehat{V}^{t}_{0}$ to be the same linear combination of the $B_{0}$-variables as $\widehat{V}^{t}$ is of the $B$-variables. Thus, similar to \eqref{eq:Vhat}, we get
\begin{align}
\widehat{V}^{t}_{0}(\pi) &= \sum_{\sigma \in X(\pi)} \frac{\rho(\sigma)}{n^{t}(\sigma)} \sum_{i = 1}^{n^{t}(\sigma)} B_{0}(\sigma, i).
\label{eq:Vhatnought}
\end{align}
The twin variables are convenient since they are all independent. However, before making use of their independence, we establish the key relationship between the deviation tendencies of $\widehat{V}^{t}$ (which is the estimator used by our algorithm) and $\widehat{V}^{t}_{0}$ (its hypothetical twin).
\begin{lemma}[$\widehat{V}^{t}$ concentrates no slower than $\widehat{V}^{t}_{0}$] For $\epsilon > 0$,
\begin{align*}
\mathbb{P}\{\widehat{V}^{t} \geq V(\pi) + \epsilon\} &\leq \mathbb{P}\{\widehat{V}^{t}_{0} \geq V(\pi) + \epsilon\};\\
\mathbb{P}\{\widehat{V}^{t} \leq V(\pi) - \epsilon\} &\leq \mathbb{P}\{\widehat{V}^{t}_{0} \leq V(\pi) - \epsilon\}.\\
\end{align*}
\label{lem:twinconcentration}
\end{lemma}
The proof of this lemma proceeds through an expansion of the moment generating functions of $\widehat{V}^{t}(\pi)$ and $\widehat{V}^{t}_{0}(\pi)$, in the manner demonstrated by Panconesi and Srinivasan~\cite[see Theorem 3.2]{doi:10.1137/panconesi93}.


\subsubsection{Deviation of $\widehat{V}^{t}_{0}$} Since $\widehat{V}^{t}_{0}$ is a linear combination of \textit{independent} random variables, we use a standard inequality to upper-bound the probability of its deviation from its expectation. Bernstein's Inequality~\cite[see Section 2.7]{Boucheron+LM:2016} is as follows.

\begin{lemma}[Bernstein's Inequality~\cite{Boucheron+LM:2016}]
Let $X_{1}, X_{2}, \dots, X_{m}$ be independent random variables with finite variance such that $|X_{i}| \leq b$ for some $b > 0$ almost surely for $1 \leq i \leq m$. Let $S = \sum_{i = 1}^{m} (X_{i} - \mathbb{E}[X_{i}])$ and $v = \sum_{i = 1}^{m} \mathbb{E}[(X_{i})^{2}]$. Then, for $\alpha > 0$, $$\mathbb{P}\{S \geq \alpha\} \leq \exp\left(- \frac{\alpha^{2} / 2}{v + \frac{b\alpha}{3}}\right).$$
\label{lem:bernsteinsinequality}
\end{lemma}
Suitable application of this lemma, detailed below, yields the following result on the deviation of $\widehat{V}^{t}_{0}$.
\begin{lemma}[Deviation of $\widehat{V}^{t}_{0}$] For $\epsilon \in (0, 1]$,
\begin{align*}
\mathbb{P}\{\widehat{V}^{t}_{0} \geq V(\pi) + \epsilon\} &\leq \exp\left(- \frac{3}{8} n^{t}(\pi) \epsilon^{2}\right);\\
\mathbb{P}\{\widehat{V}^{t}_{0} \leq V(\pi) - \epsilon\} &\leq \exp\left(- \frac{3}{8} n^{t}(\pi) \epsilon^{2}\right).
\end{align*}
\label{lem:Vnoughtconcentration}
\end{lemma}
\begin{proof}
We set up $S$, $v$, and $b$ for the application of  Lemma~\ref{lem:bernsteinsinequality}.
For $\sigma \in X(\pi)$, $1 \leq i \leq n^{t}(\sigma)$, define $$B_{1}(\sigma, i) \eqdef \frac{\rho(\sigma)}{n^{t}(\sigma)} \left(B_{0}(\sigma, i) - q(\sigma)\right).$$ There are $|\mathcal{B}|$ such bounded and mutually-independent ``$B_{1}$'' variables, where each is a linear combination of a corresponding $B_{0}$ variable. Since $|\rho(\sigma)| \leq 1$ and $|B_{0}(\sigma, i) - q(\sigma)| \leq 1$, it follows that $|B_{1}(\sigma, i)| \leq \frac{1}{n^{t}(\sigma)} \leq \frac{1}{n^{t}(\pi)} \eqdef b$. Also, we observe from \eqref{eq:Vhatnought} that the sum of the $B_{1}$-variables is $S \eqdef \widehat{V}^{t} - V(\pi)$. Finally, we have 
\begin{align*}
v &\eqdef \sum_{\sigma \in X(\pi)} \sum_{1 = 1}^{n^{t}(\pi)} \mathbb{E}\left[\left(B_{1}(\sigma, i)\right)^{2}\right] = \sum_{\sigma \in X(\pi)} \frac{q(\sigma) (1 - q(\sigma))}{n^{t}(\sigma)} \\
&\leq     
\sum_{\sigma \in X(\pi)} \frac{q(\sigma)}{n^{t}(\sigma)}
\leq     
\sum_{\sigma \in X(\pi)} \frac{q(\sigma)}{n^{t}(\pi)}
= \frac{1}{n^{t}(\pi)}.
\end{align*}
Invoking Lemma~\ref{lem:bernsteinsinequality} with these  values or bounds for $S$, $v$, and $b$, we observe that $$\mathbb{P}\{\widehat{V}^{t}_{0} \geq V(\pi) + \epsilon\} \leq \exp\left( - \frac{\epsilon^{2}/2}{\frac{1}{n^{t}(\pi)} + \frac{\epsilon}{3 n^{t}(\pi)}}\right) \leq \exp\left(- \frac{3}{8} n^{t}(\pi) \epsilon^{2}\right),$$ where we have used the fact that $\epsilon < 1$. Repeating the proof with the negations of the $B_{1}$-variables yields the same upper bound for $\mathbb{P}\{\widehat{V}^{t}_{0} \leq V(\pi) - \epsilon\}$.
\end{proof}



\subsubsection{Final step.} Our proof thus far has established the legitimacy of substituting the (possibly dependent) $B$-variables used in our algorithm with their independent, twin $B_{0}$-variables for the purpose of upper-bounding the deviation of $\widehat{V}^{t}$. We have applied Bernstein's inequality to the corresponding twin $\widehat{V}^{t}_{0}$, but notice that Lemma~\ref{lem:Vnoughtconcentration} only holds for deviations $\epsilon \in (0, 1]$. We show below that this constraint does not disrupt the claim of Theorem~\ref{thm:confidencebounds}.

The first part of the theorem considers the event $E_{1} \equiv V(\pi) \geq \text{ucb}^{t}(\pi, \delta)$, which is equivalently $\widehat{V}^{t}(\pi) \leq V(\pi) - \beta(n^{t}(\pi), \delta)$. Since $V(\pi)$ is at most $1$, and $\widehat{V}^{t}(\pi)$ is non-negative, $E_{1}$ is logically impossible if $\beta(n^{t}(\pi), \delta) > 1$. The second part of the theorem considers the event $E_{2} \equiv V(\pi) \leq \text{lcb}^{t}(\pi, \delta)$, which is equivalently $\min\{\widehat{V}^{t}(\pi), 1\} \geq V(\pi) + \beta(n^{t}(\pi), \delta)$. Notice our explicit inclusion of the ``min'' operator in the lower confidence bound. Once again, $E_{2}$ cannot possibly occur if $\beta(n^{t}(\pi), \delta) > 1$. In summary, then, it suffices to show (1) $\mathbb{P}\{ \widehat{V}^{t}(\pi) \leq V(\pi) - \beta(n^{t}(\pi), \delta) \} \leq \delta$ and (2) $\mathbb{P}\{ \widehat{V}^{t}(\pi) \geq V(\pi) + \beta(n^{t}(\pi), \delta) \} \leq \delta$, both under the condition that $\beta(n^{t}(\pi), \delta) \leq 1$. This is the precise statement of Lemma~\ref{lem:Vnoughtconcentration}, when applied with $\epsilon = \beta(n^{t}(\pi), \delta)$.


\section{Algorithms and Analysis}
\label{sec:algorithmsandnalysis}

Having established confidence bounds for policies, we apply them in algorithms for the PAC and regret-minimisation settings. Well-known algorithms \textsc{Lucb}~\cite{DBLP:conf/icml/KalyanakrishnanTAS12} and \textsc{Ucb} \cite{DBLP:journals/ml/AuerCF02} are suffixed with ``-T'' to denote their application on T-MDPs.

\subsection{\textsc{Lucb-T} for PAC Setting}
\label{subsec:algorithms-lucb-t}

\textsc{Lucb-T}, specified as Algorithm~\ref{alg:lucb-t}, is identical to the original \textsc{Lucb} algorithm~\cite{DBLP:conf/icml/KalyanakrishnanTAS12}, with each policy akin to a bandit arm. After initialising counts, for each batch $t$, \textit{two} policies are identified: one with the highest value estimate, and one with the highest upper confidence bound (UCB) among the other policies. Confidence bounds are computed for mistake probability $\delta_{L}(t) \eqdef \frac{\delta}{3|\Pi|t^{|\Sigma| + 4}}$ for each $t \geq 1$. If the two identified policies are already separated to within $\epsilon$ (line 8), the first is returned; otherwise both policies are played. By this convention, there are at most $2(t - 1)$ episodes up to batch $t$. 
\begin{algorithm}
\caption{\textsc{Lucb-T}}
\label{alg:lucb-t} 
\begin{algorithmic}[1]
    \STATE Initialise $n^{t}(\sigma), n^{t}_{+}(\sigma)$ to $0$ for $\sigma \in \Sigma$.
    \FOR{$t = 1, 2, \ldots$}
        \IF{there exists $\sigma \in \Sigma$ such that $n^{t}_{\sigma} = 0$}
            \STATE Play an arbitrary policy $\pi \in Y(\sigma)$.
        \ELSE
            \STATE $\pi^{t}_{1} \gets \argmax_{\pi \in \Pi} \widehat{V}^{t}_{\pi}$.
            \STATE $\pi^{t}_{2} \gets \argmax_{\pi \in \Pi \setminus \{\pi_{1}\}} \text{ucb}^{t}_{\pi}$.
            \IF{$\text{lcb}^{t}(\pi^{t}_{1}, \delta_{\text{L}}(t)) \geq \text{ucb}^{t}(\pi^{t}_{2}, \delta_{\text{L}}(t)) - \epsilon$}
                \RETURN $\pi^{t}_{1}$.
            \ENDIF
            \STATE Play policies $\pi^{t}_{1}$ and $\pi^{t}_{2}$.
        \ENDIF
        \STATE Set $n^{t}(\cdot)$, $n^{t}_{+}(\cdot)$ based on policies played, outcome.
    \ENDFOR
\end{algorithmic}
\end{algorithm}

\subsubsection{Efficient Implementation} For any given policy $\pi$, it is clear that $X(\pi)$, $V^{t}(\pi)$, and $\text{ucb}^{t}(\pi)$ can be computed using $\Theta(|\mathcal{S}||\mathcal{A}| + |\Sigma|)$ operations. However, notice that lines 6 and 7 in Algorithm~\ref{alg:lucb-t} require identifying policies that maximise the value estimate or the upper confidence bound (UCB). We illustrate below that although $|\Pi| = |\mathcal{A}|^{|\mathcal{S}|}$, policies $\pi^{t}_{1}$ and $\pi^{t}_{2}$ can be computed using $\text{poly}(|\mathcal{S}|, |\mathcal{A}|)$ steps.

Counts for the visits of each $(s, a, s^{\prime}) \in \mathcal{S} \times \mathcal{A} \times (\mathcal{S} \cup \Sigma)$ are updated after each episode, and the empirical transition probabilities $\widehat{p}(s, a, s^{\prime})$ are obtained by normalising. If $\widehat{p}$ is used in the RHS of \eqref{eqn:bellman}, then the optimising policy $\pi^{t}_{1}$ and its value $\widehat{V}^{t}(\pi^{t}_{1})$ are obtained bottom-up by setting as action for each state any one that maximises the corresponding RHS in \eqref{eqn:bellman}.

It gets more involved to compute $\pi^{t}_{2}$. 
In Appendix~\ref{app:pu-procedure}, we specify a procedure PU to compute a policy with the highest UCB; denote this policy $\pi^{t}_{U} \eqdef \argmax_{\pi \in \Pi} \text{ucb}^{t}(\pi, \delta_{U}(t))$. In practice, $\pi^{t}_{U}$ will usually be different from $\pi^{t}_{1}$, in which case it will itself be $\pi^{t}_{2}$. 
PU incurs $O(|\mathcal{S}| |\mathcal{A}| |k_{\max}|+ |\Sigma|)$ operations, where $k_{\max}$ is the maximum branching factor (upper-bounded by $|\mathcal{S}|$; in practice much smaller). 
Now, if it happens that $\pi^{t}_{U} = \pi^{t}_{1}$, then another bottom-up pass is performed to recursively compute a maximum-UCB policy that is different from $\pi^{t}_{1}$. 
The basis of the recursion is that the policy following a state-action pair must be identical to $\pi^{t}_{1}$ at all but one child, and different from $\pi^{t}_{1}$ for exactly one child. 
This process incurs essentially the same computation as PU. 
Our submission includes full code for all our algorithms (Appendix~\ref{app:code-hyperparameter}), along with the test environments described in Section~\ref{sec:experiments}.

\subsubsection{Correctness.}
\label{subsubsec:algorithms-lucb-t-correctness}

The correctness of \textsc{Lucb-T} follows from a union bound over mistake probabilities.
\begin{proposition}
\label{prop:lucb-t-correctness}
The probability that \textsc{Lucb-T} returns a policy $\pi \in \Pi$ such that $V(\pi) < V(\pi^{\star}) - \epsilon$ is at most $\delta$.    
\end{proposition}
\begin{proof}
A non-$\epsilon$-optimal policy can be returned only on the \textit{bad} event that (1) on some batch $t \geq 1$, (2) for some sequence of play counts $(n^{t}(\sigma))_{\sigma \in \Sigma}$, and (3) for some policy $\pi \in \Pi$, the interval $[\text{lcb}^{t}(\pi, \delta_{L}(t)), \text{ucb}^{t}(\pi, \delta_{L}(t))]$ does \textit{not} contain $V(\pi)$. On batch $t$, each play count must be between $1$ and $t$. Applying Theorem~\ref{thm:confidencebounds}, the probability of the bad event is at most
$\sum_{t = 1}^{\infty} t^{|\Sigma|} |\Pi|  2 \delta_{L}(t) \leq \delta.$ 
\qedhere
\end{proof}

\subsubsection{Sample Complexity.}
\label{subsubsec:algorithms-lucb-t-samplecomplexity}

Sample-complexity analysis proceeds in the same manner as of \textsc{Lucb}~\cite{DBLP:conf/icml/KalyanakrishnanTAS12}, beginning with the definition of instance-specific ``gaps''. Define 
\begin{align*}
V_{2} &\eqdef \max_{\pi \in \Pi_{\text{opt}} \setminus \{\pi^{\star}\}} V(\pi),\\
\Delta_{\pi} &\eqdef \begin{cases} V(\pi^{\star}) - V_{2} & \text{if } \pi = \pi^{\star},\\
V(\pi^{\star}) - V(\pi) & \text{if } \pi \neq \pi^{\star},\end{cases} \text{ and}\\
\Delta^{\epsilon}_{\pi} &\eqdef \max\{\Delta_{\pi}, \epsilon\} \text{ for } \epsilon > 0.
\end{align*}
For analysing \textsc{Lucb-T}, we additionally define for $\sigma \in \Sigma$, $$
\Delta^{\epsilon}_{\sigma} \eqdef \min_{\pi \in Y(\sigma)} \Delta^{\epsilon}_{\pi}.$$
We upper-bound the sample complexity of \textsc{Lucb-T} as follows.
\begin{theorem}
\label{thm:lucb-t-sc}
For $\epsilon, \delta \in (0, 1)$, the expected number of episodes taken by \textsc{Lucb-T} before termination is
$$\mathcal{O}\left( \sum_{\sigma \in \Sigma} \frac{|\Sigma|}{(\Delta^{\epsilon}_{\sigma})^{2}} \left( \log \frac{1}{\delta} + \log \sum_{\sigma \in \Sigma} \frac{1}{(\Delta^{\epsilon}_{\sigma})^{2}} + |\mathcal{S}| \log |\mathcal{A}|\right)\right).$$
\end{theorem}
\begin{proof}
The proof follows the same template as that of \textsc{Lucb}~\cite{DBLP:conf/icml/KalyanakrishnanTAS12}. A policy $\pi$ is called \textit{needy} on episode $t$ if its play count is smaller than a constant times $\frac{1}{(\Delta^{\epsilon}_{\pi})^{2}} \ln \frac{1}{\delta_{L}(t)}$. A terminal state $\sigma \in \Sigma$ is called needy if it is the least-played state of some needy policy $\pi \in Y(\sigma)$. Notice that \textit{non-needy} policies have sufficiently small width $\beta$. It is shown that if neither of the policies played on episode $t$ are needy, then some policy must have violated its upper or lower confidence bound. By a union bound similar to that used in Proposition~\ref{prop:lucb-t-correctness}, the probability of such an event is at most a constant times $\frac{\delta}{t^{3}}$.

On the other hand, if indeed some needy policy is played on episode $t$, it means that some needy terminal state $\sigma \in \Sigma$ is played on episode $t$. Since there is a cap on the number of episodes in which $\sigma$ can both be needy and get played, the total number of ``good'' episodes (which play needy $\sigma$'s) is in the order of $\sum_{\sigma \in \Sigma} \frac{1}{(\Delta^{\epsilon}_{\sigma})^{2}} \ln \frac{1}{\delta_{L}(t)}$ for sufficiently large $t \geq t^{\star}$. For an appropriate choice of $t^{\star}$, the probability of not stopping at or before $t^{\star}$ is at most $\frac{\delta}{(t^{\star})^{2}}$---and this property implies the claimed upper bound. 
\end{proof}

\subsubsection{Tighter Upper Bound.}
\label{subsubsec:tighter-upper-bound}

Notice that $|\Sigma|$ factor in the sample-complexity upper bound from Theorem~\ref{thm:lucb-t-sc}. This is an artefact of the union bound over variable play counts, which necessitated a $\frac{1}{t^{|\Sigma|}}$ factor in $\delta_{L}(t)$. An algorithmic change to \textsc{Lucb-T} can remove the $|\Sigma|$ factor. Rather than use all the available samples for each $\sigma \in X(\pi)$, we can use only the first $n^{t}(\pi)$ samples for each (possibly ignoring a lot of data for some terminal states). Since $n^{t}(\pi)$ has to be between $1$ and $t$, the union is only over $t$ (rather than $t^{|\Sigma|}$) events. In view of the same amount of data being used for each terminal state, we refer to this algorithmic variant as \textsc{Lucb-T-Uniform}.

Whereas \textsc{Lucb-T} only needs to store the counts $n^{t}(\sigma)$ and 
$n^{t}_{+}(\sigma)$ for $\sigma \in \Sigma$, \textsc{Lucb-T-Uniform} needs to store the entire sequence of outcomes (0's and 1's) from the plays of each $\sigma \in \Sigma$. Consequently, while \textsc{Lucb-T} uses $\text{polylog}(t)$ memory, \textsc{Lucb-T-Uniform} needs $\Theta(t)$ memory. Thus, while the latter algorithm yields a superior upper bound, it is less convenient in practice, and in fact even performs worse on larger problem instances (coming up in Section~\ref{sec:experiments}).

\subsection{\textsc{Ucb-T} for Regret-minimisation Setting}
\label{subsec:algorithms-ucb-t}

The same idea of viewing policies as bandit arms gives rise to \textsc{Ucb-T} (Algorithm~\ref{alg:ucb-t}), which implements the \textsc{Ucb} algorithm \cite{DBLP:journals/jmlr/Auer02} for bandits. In this case, after initialising plays, a policy with the highest upper confidence bound is played on each episode. For episode $t \geq 1$, confidence bounds are for mistake probability $\delta_{U}(t) \eqdef \frac{1}{|\Pi| t^{|\Sigma| + 4}}.$
\begin{algorithm}
\caption{\textsc{Ucb-T}}
\label{alg:ucb-t} 
\begin{algorithmic}[1]
    \STATE Initialise $n^{t}(\sigma), n^{t}_{+}(\sigma)$ to $0$ for $\sigma \in \Sigma$
    \FOR{$t = 1, 2, \ldots$}
        \IF{there exists $\sigma \in \Sigma$ such that $n^{t}_{\sigma} = 0$}
            \STATE Play an arbitrary policy $\pi \in Y(\sigma)$.
        \ELSE
            \STATE $\pi^{t} \gets \argmax_{\pi \in \Pi} \text{ucb}^{t}(\pi, \delta_{U}(t))$.
            \STATE Play policy $\pi^{t}$.
        \ENDIF
        \STATE Set $n^{t}(\cdot)$, $n^{t}_{+}(\cdot)$ based on policy played, outcome.
    \ENDFOR
\end{algorithmic}
\end{algorithm}

The computation of the policy maximising the UCB (line 6) is done by the PU procedure described in Appendix~\ref{app:pu-procedure}.

\subsubsection{Regret.}
\label{subsec:algorithms-analysis-ucb-t-regret}

The structure put forth by \citet{DBLP:journals/ml/AuerCF02} to upper-bound the regret of \textsc{Ucb} also extends to $\textsc{Ucb-T}$. The same can also be interpreted through the terminology of ``needy'' policies and terminal states, presented in the proof of Theorem~\ref{thm:lucb-t-sc}, but with two differences. First, only the plays of non-optimal policies---the elements of $\Pi \setminus \Pi_{\text{opt}}$---contribute to the regret. Second, for $\sigma \in \Sigma$, we conservatively upper-bound the number of needy plays of $\sigma$ by assuming it is played by the policy in $Y(\sigma) \setminus \Pi_{\text{opt}}$ with the \textit{smallest} gap. Conservatively, each such play still contributes regret according to the policy in $Y(\sigma) \setminus \Pi_{\text{opt}}$ with the \textit{largest} gap. Formally, for $\sigma \in \Sigma$, 
\begin{align*}
\Delta^{\min}_{\sigma} &\eqdef \min_{\pi \in Y(\sigma) \setminus \Pi_{\text{opt}}} \Delta_{\pi};~~~~
\Delta^{\max}_{\sigma} &\eqdef \max_{\pi \in Y(\sigma) \setminus \Pi_{\text{opt}}} \Delta_{\pi}.
\end{align*}
We obtain the following upper bound on regret.
\begin{theorem}
\label{thm:ucb-t-regret}
There exists $c > 0$ such that for $T \geq 2$, the regret $R_{T}$ of the \textsc{Ucb-T} algorithm satisfies $$R_{T} \leq c \cdot \sum_{\sigma \in \Sigma, Y(\sigma) \not\subseteq \Pi_{\text{opt}}} \left( \frac{|\Sigma| \Delta^{\max}_{\sigma}}{(\Delta^{\min}_{\sigma})^{2}} \ln T + \ln |\Pi| \right).$$
\end{theorem}

The detailed proof is given in Appendix~\ref{app:proof-ucb}. As discussed in Section~\ref{subsubsec:tighter-upper-bound}, the $|\Sigma|$ factor can be removed, at the expense of additional memory and compute. In view of its limited practical value, we skip this variant for regret-minimisation.

\section{Experimental Evaluation}
\label{sec:experiments}

\begin{table*}[b]
\caption{PAC stopping times for player \texttt{x} in 3-card and 5-card Kuhn Poker. Results average $10$ runs, and show one standard error.}
\Description{PAC stopping times for player \texttt{x} in 3-card and 5-card Kuhn Poker. Results average $10$ runs, and show one standard error.}
\label{tab:x_kuhn_stopping_times}
\vspace{-0.3cm}
\centering
\begin{tabular}{@{}lcccc@{}}
\toprule
\multirow{2}{*}{\textbf{Algorithm}} & \multicolumn{4}{c}{\textbf{Experiment Parameters}} \\
\cmidrule(l){2-5}
& \multicolumn{2}{c}{\textbf{3-card Kuhn Poker}} & \multicolumn{2}{c}{\textbf{5-card Kuhn Poker}} \\
\cmidrule(l){2-3} \cmidrule(l){4-5}
& \multicolumn{1}{c}{player \texttt{x}, $\epsilon = 0.05, \delta = 0.05$} & \multicolumn{1}{c}{player \texttt{x}, $\epsilon = 0.1, \delta = 0.1$} & \multicolumn{1}{c}{player \texttt{x}, $\epsilon = 0.05, \delta = 0.05$} & \multicolumn{1}{c}{player \texttt{x}, $\epsilon = 0.1, \delta = 0.1$} \\ \midrule
\textsc{Lucb} & $\mathbf{0.473 \times 10^6 \pm 5.8\%}$ & $\mathbf{0.114 \times 10^6 \pm 7.3\%}$ & $7.947 \times 10^6 \pm 2.3\%$ & $2.138 \times 10^6 \pm 2.2\%$ \\
\textsc{Lucb-T} & $2.170 \times 10^6 \pm 5.4\%$ & $0.531 \times 10^6 \pm 4.8\%$ & $\mathbf{3.637 \times 10^6 \pm 3.7\%}$ & $\mathbf{0.942 \times 10^6 \pm 3.5\%}$ \\
\textsc{Lucb-T-Uniform} & $2.117 \times 10^6 \pm 4.5\%$ & $0.508 \times 10^6 \pm 6.0\%$ & $4.880 \times 10^6 \pm 8.7\%$ & $1.194 \times 10^6 \pm 5.5\%$ \\ \bottomrule
\end{tabular}
\end{table*}

We evaluate our algorithms on two standard benchmark games, Kuhn Poker \cite{kuhnpoker} and Leduc Poker \cite{BayesBluffOppModelling}, as well as a third game, Reconnaissance Blind Tic Tac Toe (RBT) \cite{rbt}. Full descriptions are given in Appendix \ref{app:gamedesc}. All three games are 2-player, zero-sum, with hidden information. We refer to the first and second players as $\mathtt{x}$ and $\mathtt{o}$. Each ``state'' in these games is an action-observation history. Kuhn Poker and Leduc Poker, common benchmarks for Nash-equilibrium computation, have $6$ and $144$ states, respectively. RBT, designed as a smaller version of Reconnaissance Blind Chess, has roughly $10^{7}$ states for $\mathtt{x}$ and $2 \times 10^7$ for $\mathtt{o}$. To the best of our knowledge, our results are the first to demonstrate the feasibility of learning on RBT.

\subsection{Benefit of Sharing Data}

On the smaller Kuhn Poker game, we compare \textsc{Lucb-T} and \textsc{Lucb-T-Uniform} with vanilla
\textsc{Lucb}~\cite{DBLP:conf/icml/KalyanakrishnanTAS12}, which treats
each deterministic policy as a separate bandit arm, with no data being shared across arms. Table~\ref{tab:x_kuhn_stopping_times} shows the average stopping times when player \texttt{x} learns to play agains \texttt{o}'s equilibrium strategy (complementary results are in appendix \ref{app:kuhnexp}). Notice that indeed \textsc{Lucb} is the most economical. The slack in the analysis of 
\textsc{Lucb-T} explains this observation; Kuhn Poker is a small game: with $3$ cards there are only 64 policies.
We devise a 5-card generalisation of Kuhn Poker (refer to appendix \ref{app:gamedesc}) which has 1024 deterministic policies. On this variant, it becomes apparent that $\textsc{Lucb}$ (whose sample complexity scales exponentially in the number of states) falls behind the ``tree'' variants.
This shortcoming precludes the use of \textsc{Lucb} in larger games. Interestingly, we also notice that \textsc{Lucb-T-Uniform}, which outperforms \textsc{Lucb-T} on the 3-card version, performs worse on the $5$-card version. Ignoring a lot of informative data hurts \textsc{Lucb-T-Uniform} in practice as problem sizes get larger.



\subsection{Comparisons with Baselines on Larger Games}

\begin{figure*}[t]
    \centering
    \subfloat[Leduc Poker PAC ($\mathtt{x}$)]{%
        \includegraphics[width=0.248\linewidth]{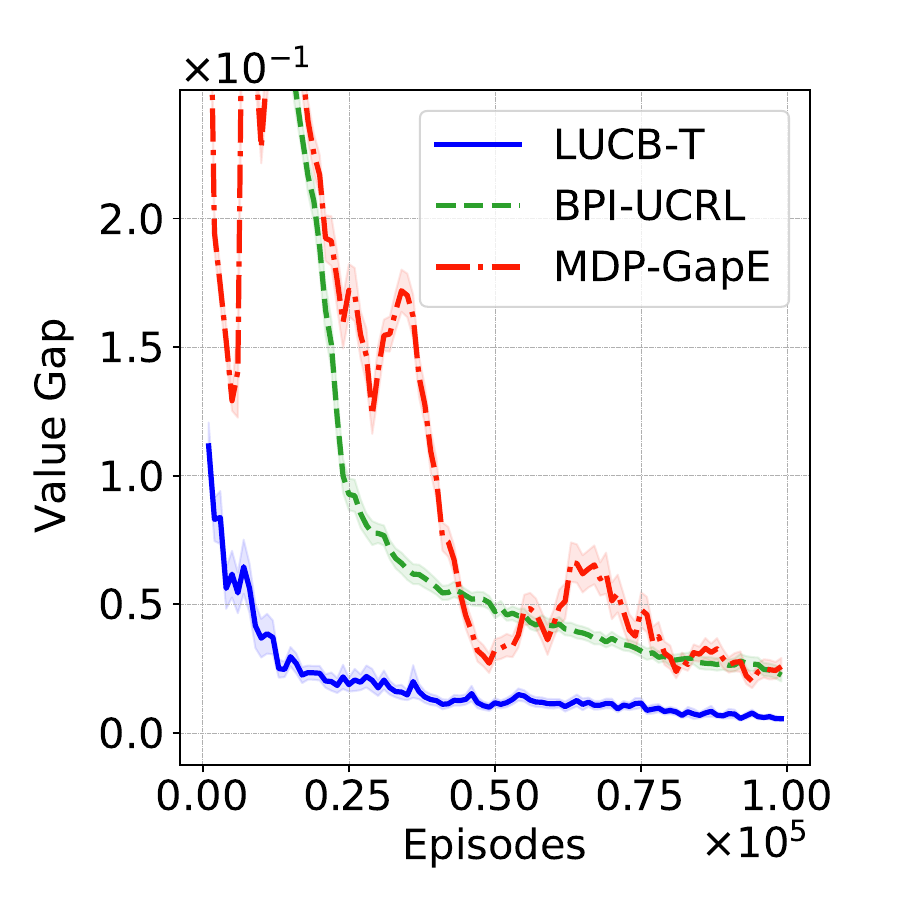}%
        \label{fig:leduc_pac_x_combined_singlefig}%
    }
    \subfloat[Leduc Poker PAC ($\mathtt{o}$)]{%
        \includegraphics[width=0.248\linewidth]{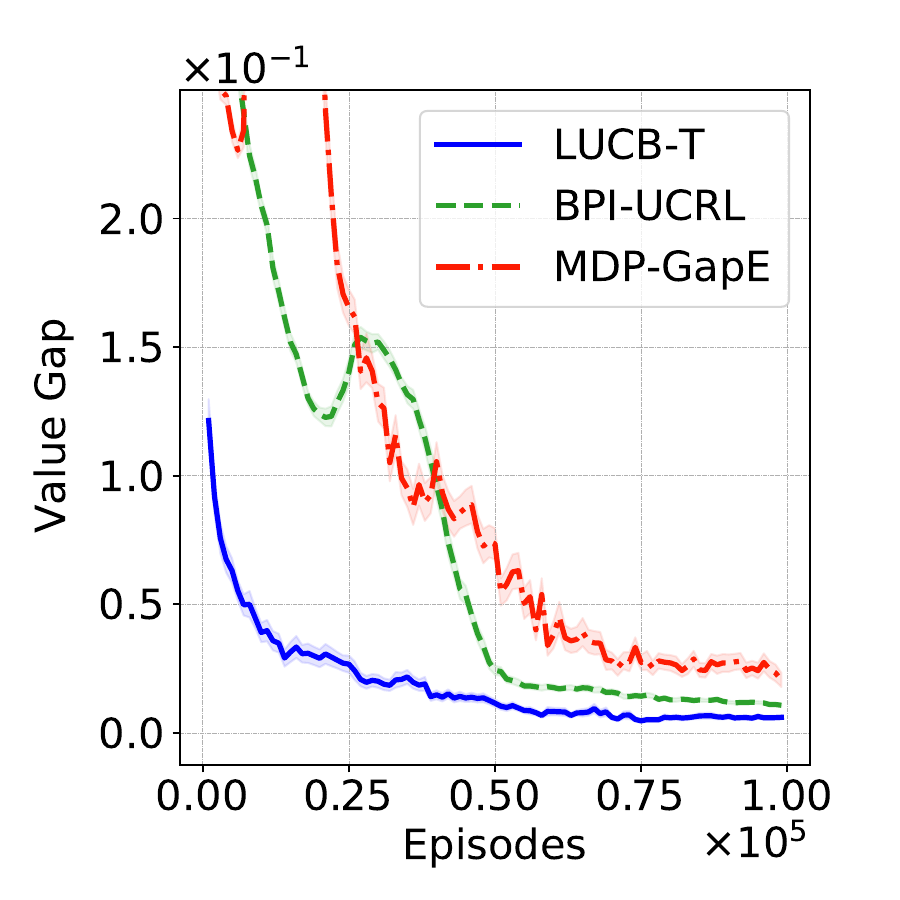}%
        \label{fig:leduc_pac_o_combined_singlefig}%
    }
    \subfloat[RBT PAC ($\mathtt{x}$)]{%
        \includegraphics[width=0.248\linewidth]{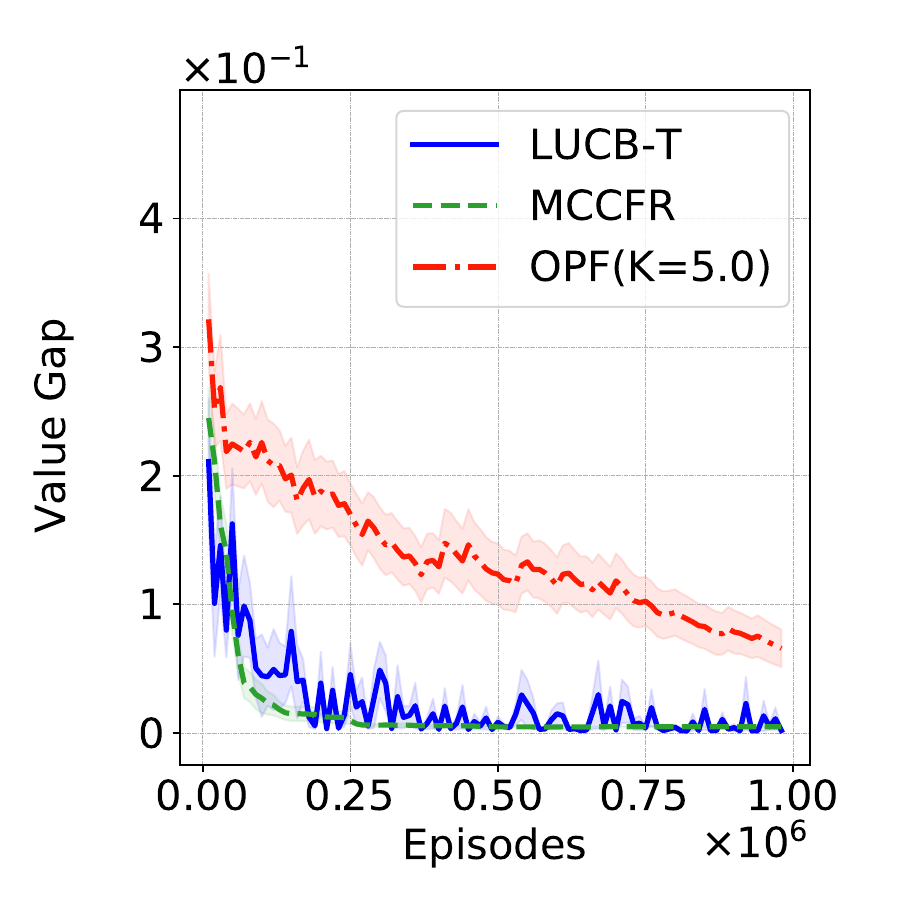}%
        \label{fig:rbt_pac_x_combined_singlefig}%
    }
    \subfloat[RBT PAC ($\mathtt{o}$)]{%
        \includegraphics[width=0.248\linewidth]{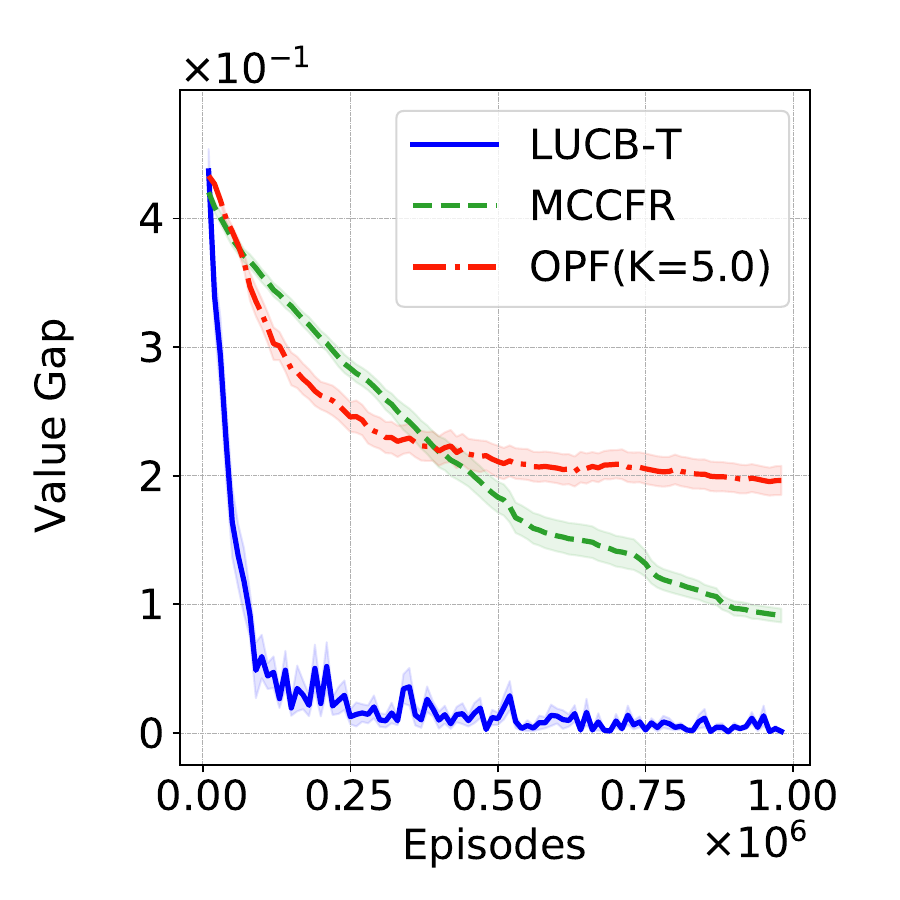}%
        \label{fig:rbt_pac_o_combined_singlefig}%
    }

\vspace{-0.4cm}

    \subfloat[Leduc Poker Regret ($\mathtt{x}$)]{%
        \includegraphics[width=0.248\linewidth]{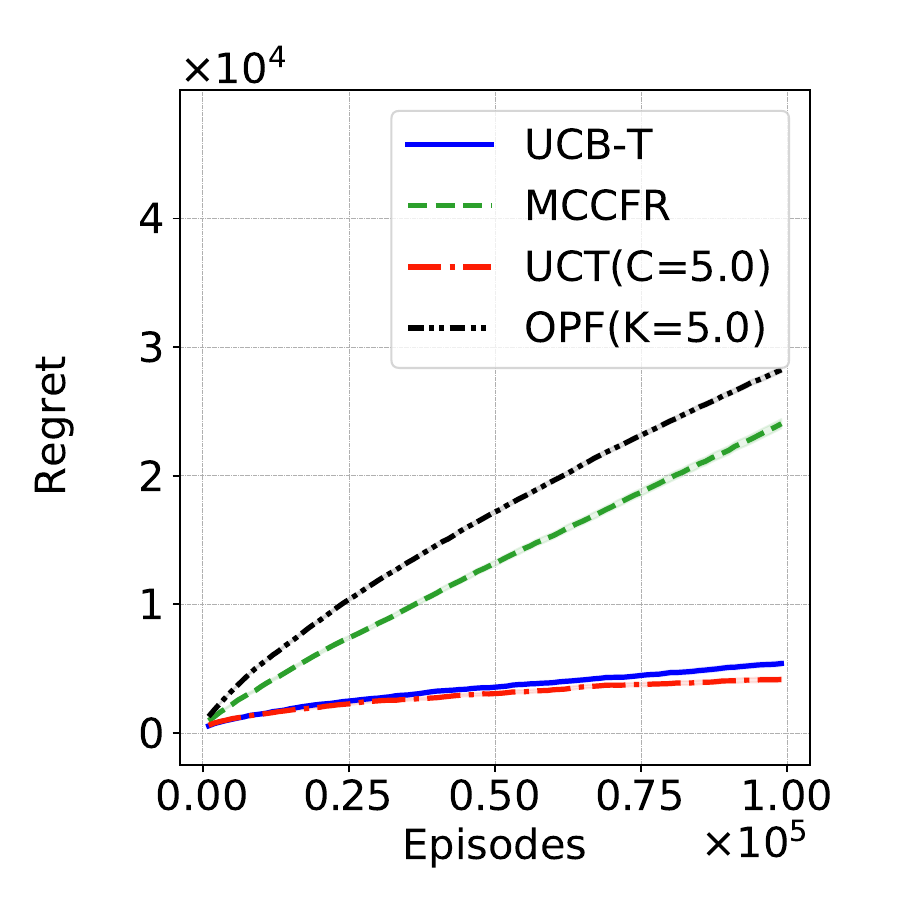}%
        \label{fig:leduc_regret_x_combined_singlefig}%
    }
    \subfloat[Leduc Poker Regret ($\mathtt{o}$)]{%
        \includegraphics[width=0.248\linewidth]{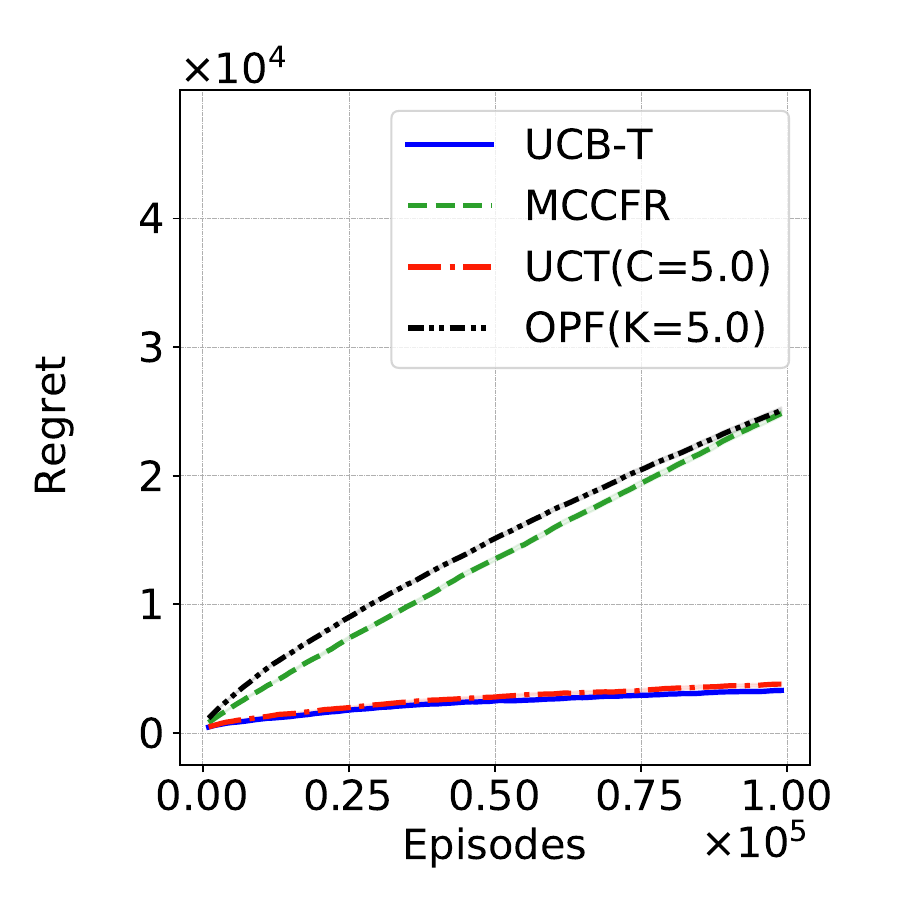}%
        \label{fig:leduc_regret_o_combined_singlefig}%
    }
    \subfloat[RBT Regret ($\mathtt{x}$)]{%
        \includegraphics[width=0.248\linewidth]{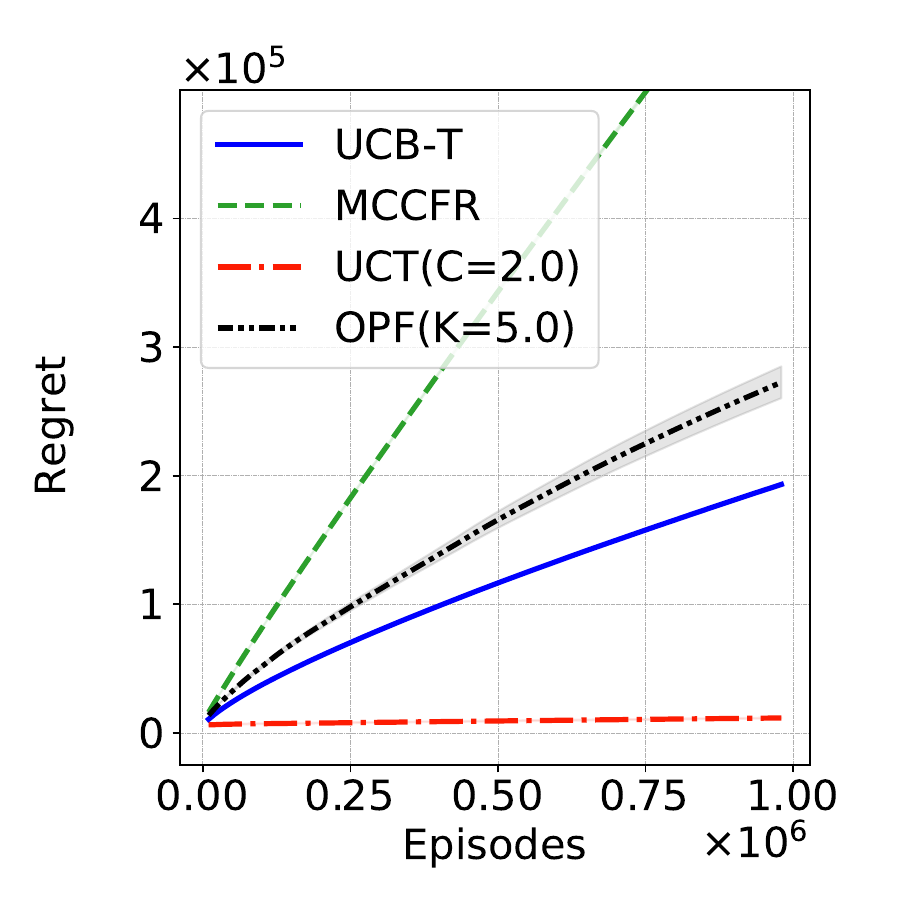}%
        \label{fig:rbt_regret_x_combined_singlefig}%
    }
    \subfloat[RBT Regret ($\mathtt{o}$)]{%
        \includegraphics[width=0.248\linewidth]{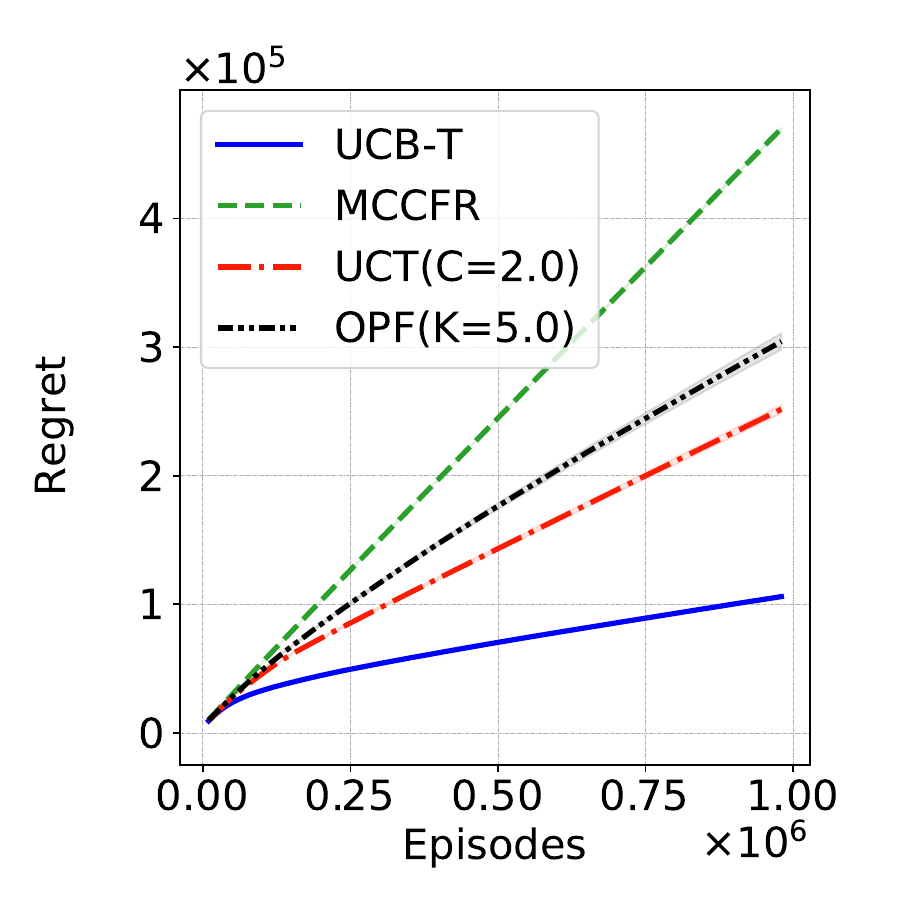}%
        \label{fig:rbt_regret_o_combined_singlefig}%
    }
\vspace{-0.2cm}
    \caption{Performance for Leduc Poker and RBT. The top row (a-d) corresponds to the PAC setting, and  the bottom row (e-h) to regret minimisation results. All plots are averaged over 50 experiments for Leduc Poker and 25 experiments for RBT, showing outcomes for both players x and o. Error bars correspond to one standard error.}
    \label{fig:leduc_rbt_combined_results}
    \Description{This figure shows experiments for Leduc Poker and RBT, for the PAC and regret paradigms. }
\end{figure*}


Much of the literature on exploration in MDPs has not empirically evaluated or provided code for algorithms, which, even if implemented, do not perform well with default hyperparameter settings.
We were able to implement and fine-tune working versions of \textsc{Bpi-Ucrl} \cite{DBLP:conf/alt/TirinzoniMK23} and \textsc{Mdp-GapE} \cite{MDPGAPE}, which we use as baselines for comparison with our PAC algorithm, \textsc{Lucb-T}. For regret minimisation, the algorithms with the best instance-dependent theoretical bounds are \textsc{Mvp} \cite{zhang24Settling}, \textsc{Amb} \cite{FineGrainedGapRegret} and \textsc{StrongEuler} \cite{BeyondValueFunctionGapsRegret}, for which we were unable to source working implementations. However, algorithms in the games literature have been implemented and tested on a number of different applications. We compare \textsc{Ucb-T} with \textsc{Mccfr} \cite{lanctot2009monte}, On-path flipping (\textsc{Opf}) \cite{farina2021model}, and \textsc{Uct} \cite{DBLP:conf/ecml/KocsisS06}. 
Most of these baselines have hyperparameters to balance between exploration and exploitation. We use simplified confidence bounds in our implementations of \textsc{Lucb-T}, \textsc{Ucb-T}, \textsc{Bpi-Ucrl} and \textsc{Mdp-GapE}. We have tuned these hyperparameters and bounds to the best of our ability for optimising performance. Implementation-related details are given in Appendix~\ref{app:code-hyperparameter}. 



Figure \ref{fig:leduc_rbt_combined_results} shows performance plots
for the larger games of Leduc Poker and RBT. In all the games we fix an $\epsilon$-Nash policy for the opponent player, computed using \textsc{Cfr+} \cite{tammelin2014solving}. 
For the PAC setting, we follow the standard practice of showing the ``value gap'' $\mathbb{E}[V(\pi^{\star}) - V(\pi^{t}_{1})]$ as learning proceeds~\cite{DBLP:conf/colt/KaufmannK13}, since stopping times are prohibitively large. For \textsc{Lucb-T} we use the same $\delta$ value for all our experiments. We observe consistently competitive performance for our algorithms across all three games for both players, for both the PAC and regret paradigms. Our algorithms scale well with game size, widening the gap between other algorithms on large problem instances, with RBT player $\mathtt{o}$ (Figure~\ref{fig:rbt_pac_o_combined_singlefig}) especially notable. For RBT, since the PAC algorithms \textsc{Bpi-Ucrl} and \textsc{Mdp-GapE} did not perform well, we have instead plotted a comparison with the regret minimising algorithms (figures~\ref{fig:rbt_pac_x_combined_singlefig} and \ref{fig:rbt_pac_o_combined_singlefig}). Although our algorithms dominate all others that have provable performance guarantees, we notice that the popular $\textsc{Uct}$ algorithm incurs lower regret on smaller problems. By publishing our code, we hope to attract more attention to the empirical evaluation of on-line learning algorithms for games, with the aim of reconciling gaps between theory and practice.


\section{Related Work and Discussion}
While we are not aware of previous work specifically tailored to T-MDPs, there is indeed a vast literature on on-line learning in MDPs, as well as (imperfect-information) games. Unlike our contribution that benefits algorithms both for the PAC and the regret-minimisation settings, most previous work is tailored to one or the other. Specific goals have included {\em minimax PAC bounds}~\cite{DBLP:conf/icml/Dann0WB19, DBLP:conf/icml/MenardDJKLV21}, {\em instance-dependent PAC bounds} \cite{MDPGAPE,DBLP:conf/nips/WagenmakerJ22,DBLP:conf/alt/TirinzoniMK23}, {\em PAC RL with a generative model}~\cite{DBLP:journals/ml/AzarMK13,DBLP:conf/nips/ZanetteKB19, 10.5555/3709347.3743628}, {\em reward-free exploration}~\cite{DBLP:conf/icml/JinKSY20, pmlr-v132-kaufmann21a}, {\em minimax regret bounds}~\cite{DBLP:conf/icml/Dann0WB19, zhang24Settling}, and {\em instance-dependent regret bounds} \cite{NonAsymptoticGapDependentRegret, BeyondValueFunctionGapsRegret,chen2025sharpgapdependentvarianceawareregret}. 



In the PAC setting, the theoretical results most relevant to ours have focused on furnishing sample-complexity upper bounds~\cite{MDPGAPE, DBLP:conf/colt/WagenmakerSJ22, DBLP:conf/alt/TirinzoniMK23}. For these algorithms, the sample complexity upper bound typically takes the form of $\Tilde{\mathcal O}(C(\MM, \epsilon)\log (\nicefrac{1}{\delta}))$, where $C(\MM, \epsilon)$ is an instance-dependent quantity capturing the hardness of learning. These bounds invariably have an inverse dependence on the probability that a state will be visited ($q(\sigma)$ in our work)---which can be arbitrarily small. Whereas our assumption of known rewards lets us avoid this dependence, available \textit{lower bounds} suggest that it cannot be avoided in general MDPs~\cite{DBLP:journals/corr/abs-2311-05638lowerbound}. It is also to be mentioned that upper bounds across different analyses remain incomparable to due varying definitions of the value gaps ($\Delta_{\sigma}$ in our work). We compile a brief technical survey of this line of results in Appendix~\ref{app:instance-dependen-pac-bounds}.

A related topic in this context pertains to ``optimistic'' algorithms, in which the sampling rule selects actions at each state greedily with respect to some UCB quantity.
\citet{DBLP:conf/colt/WagenmakerSJ22} show that optimistic algorithms fail to achieve the instance-optimal sample complexity bound, and that in general more aggressive exploration is necessary. As yet, we are unaware of a sample-complexity lower bound specific to T-MDPs.

Formal bounds on the \textit{regret} of on-line learning algorithms for MDPs are of the form $\mathcal{C}(\mathcal{M)}\log T$, where $\mathcal{C}(\mathcal{M)}$ is a quantity that typically depends on \textit{sub-optimality gaps} as well as  variance-related terms~\cite{NonAsymptoticGapDependentRegret,BeyondValueFunctionGapsRegret}.
An important result of \citet{NonAsymptoticGapDependentRegret} is that optimistic algorithms must incur an additional regret term that depends inversely on $\Delta_{\texttt{min}}$ \textemdash the minimum gap among all state-action pairs. \citet{FineGrainedGapRegret} show that non-optimistic algorithms can reduce the unfavourable dependency on $\Delta_{\texttt{min}}$ by introducing the \textsc{Amb} algorithm, which eliminates this dependency in MDPs with a single optimal action at each state. 
Neither \textsc{Lucb-T} nor \textsc{Ucb-T} is an optimistic algorithm. Although the latter is greedy with respect to a UCB, the UCB itself is for an entire policy (rather than a local state-action pair), computed 
using the PU procedure (Section~\ref{subsec:algorithms-lucb-t}, Appendix~\ref{app:pu-procedure}).



The literature on imperfect-information extensive form games (IIEFGs)~\cite{kuhn1953extensive, shoham2008multiagent} has predominantly focused on the two-player zero-sum setting.
The most common solution concept for IIEFGs is the Nash equilibrium, which is an assignment of strategies to the players such that neither player can gain by unilaterally deviating. Since the exact computation of Nash equilibria is expensive, the preferred alternative is iterative approaches to compute $\epsilon$-Nash equilibria.
The most popular class of algorithms in this regard are from the regret minimisation paradigm. In particular, counterfactual regret minimisation (\textsc{Cfr})~\cite{zinkevich2007regret} and its variants~\cite{lanctot2009monte, tammelin2014solving, brown2017dynamic, brown2019deep, farina2021model} are state-of-the-art algorithms 
that have led to the development of superhuman agents for large IIEFGs such as Heads-Up Limit Texas Hold’em Poker and Heads-Up No Limit Texas Hold’em Poker~\cite{HULHESolved, moravvcik2017deepstack, poker_headsup_libr, brown2019superhuman}. The  problem we address in this paper is a degenerate type of game, in which one player is fixed, and so algorithms such as \textsc{Mccfr} \cite{lanctot2009monte} and \textsc{Opf} \cite{farina2021model} are guaranteed to converge to the best response. This behaviour of theirs is indeed apparent from figures~\ref{fig:leduc_rbt_combined_results}(e)-(h), but the rates of convergence are much slower than \textsc{Ucb-T}, which has specifically been designed for T-MDPs. A theoretical counterpart to this empirical observation is that the regret and sample-complexity bounds 
from the games literature~\cite{bai2022near} are \textit{worst-case} (in terms of $\epsilon$) rather than instance-specific (in terms of value gaps $\Delta_{\sigma}$).

\section{Conclusion and Future Work}

In this paper, we have generalised well-known bandit algorithms \textsc{Lucb} and \textsc{Ucb} to T-MDPs. Our main tool is a concentration inequality (the basis for Theorem~\ref{thm:confidencebounds}) that facilitates the values of multiple policies (an exponentially-sized set) to be estimated simultaneously from a common pool of (polynomially-sized) data. The resulting confidence bound fits naturally into both \textsc{Lucb} and \textsc{Ucb}, and enables polynomial (in time and memory) computation at each decision making step. Our confidence bound may be of independent interest for other analyses involving learning in sequential tasks.

We present theoretical support for our algorithms through upper bounds on sample complexity and regret. We also present supporting empirical results from three IIEFGs, on the task of learning a best response against a fixed opponent.
RBT is an especially promising benchmark due to its much larger scale (10--20 million states) compared to current alternatives.




\begin{acks}
We thank Manas Thakur for helping us parallelize our \textsc{Cfr+} implementation for RBT. 
\end{acks}
\clearpage



\bibliographystyle{ACM-Reference-Format} 
\balance
\bibliography{sample}


\clearpage
\onecolumn
\appendix
\section{Proof of Proposition~\ref{prop:v-terminal}}
\label{app:proof-v-terminal}

The proposition states that for $\pi \in \Pi$,  \begin{align*}
V^{\pi}(s_{1}) = \sum_{\sigma \in X(\pi)} q(\sigma) \cdot \rho(\sigma).   
\end{align*}

For any internal node $\bar{s} \in S$ in our tree, let $\rho^{\bar{s}}(s)$ and $q^{\bar{s}}(s)$ denote the respective quantities defined in \eqref{eqn:qrho},
but on the subtree rooted at $\bar{s}$. Thus, $\rho(s) = \rho^{s_{1}}(s)$, and $q(s) = q^{s_{1}}(s)$. We shall show that for all $\bar{s} \in S$,
\begin{align}
V^{\pi}(\bar{s}) &= \sum_{\sigma \in \Sigma} q^{\bar{s}}(\sigma) \cdot \rho^{\bar{s}}(\sigma),
\label{eq:prop-proof-claim}
\end{align}
with the convention that $q^{\bar{s}}(\sigma) \eqdef 0$ if $\sigma$ is not reachable from $\bar{s}$ by $\pi$. Hence \eqref{eq:prop-proof-claim} implies the statement of the proposition.

The proof is by induction, with the base case being states $\bar{s}$ at level $\mathcal{H}$. For each $\sigma$ that is reachable from $\bar{s}$ by taking $\pi(\bar{s})$ (with positive probability $p(\bar{s})$), we have $q^{\bar{s}}(\sigma) = p(\bar{s}, \pi(\bar{s}), \sigma)$ and $\rho^{\bar{s}}(\sigma) = r(\bar{s}, \pi(\bar{s}), \sigma)$. Hence \eqref{eq:prop-proof-claim} becomes identical to the Bellman equation in \eqref{eqn:bellman}, thereby establishing the base case.

Now consider the case that $\bar{s}$ is at level $h \in \{1, 2, \dots, \mathcal{H} - 1\}$. We assume that the claim is true for all states at level $h + 1$. Hence, expanding \eqref{eq:lcb}, we get
\begin{align*}
V^{\pi}(\bar{s}) &= \sum_{s^{\prime} \in \mathcal{S} \cup \Sigma} p(\bar{s}, \pi(\bar{s}), s^{\prime}) \left( r(\bar{s}, \pi(\bar{s}), s^{\prime}) + \gamma V^{\pi}(s^{\prime})\right)\\
&= \sum_{s^{\prime} \in \mathcal{S} \cup \Sigma} p(\bar{s}, \pi(\bar{s}), s^{\prime}) \left( r(\bar{s}, \pi(\bar{s}), s^{\prime}) + \gamma \sum_{\sigma \in \Sigma} q^{s^{\prime}}(\sigma) \rho^{s^{\prime}}(\sigma) \right)\\
&= \sum_{s^{\prime} \in \mathcal{S} \cup \Sigma} p(\bar{s}, \pi(\bar{s}), s^{\prime})  \sum_{\sigma \in \Sigma} q^{s^{\prime}}(\sigma) \left(r(\bar{s}, \pi(\bar{s}), s^{\prime}) + \gamma \rho^{s^{\prime}}(\sigma) \right) \\
&= \sum_{\sigma \in \Sigma} \sum_{s^{\prime} \in \mathcal{S} \cup \Sigma} p(\bar{s}, \pi(\bar{s}), s^{\prime})   q^{s^{\prime}}(\sigma) \left(r(\bar{s}, \pi(\bar{s}), s^{\prime}) + \gamma \rho^{s^{\prime}}(\sigma) \right),
\end{align*}
where we have used the fact that $\sum_{\sigma \in \Sigma} q^{s^{\prime}}(\sigma) = 1$. By the tree property, each $\sigma \in \Sigma$ has at most one path from $\bar{s}$. If $\sigma$ has such a path, let $s^{\prime}(\sigma)$ be the child of $\bar{s}$ along this path. Hence
\begin{align*}
V^{\pi}(\bar{s}) &= \sum_{\sigma \in \Sigma} p(\bar{s}, \pi(\bar{s}), s^{\prime}(\sigma))   q^{s^{\prime}(\sigma)}(\sigma) \left(r(\bar{s}, \pi(\bar{s}), s^{\prime}(\sigma)) + \gamma \rho^{s^{\prime}(\sigma)}(\sigma) \right).
\end{align*}
By the definitions in \eqref{eqn:qrho}, $r(\bar{s}, \pi(\bar{s}), s^{\prime}(\sigma)) + \gamma \rho^{s^{\prime}}(\sigma) = \rho^{\bar{s}}(\sigma)$
and
$p(\bar{s}, \pi(\bar{s}), s^{\prime}(\sigma)) q^{s^{\prime}(\sigma)}(\sigma) = q^{\bar{s}}(\sigma)$. Substituting these values above, we obtain \eqref{eq:prop-proof-claim}.

\section{Proof of Lemma~\ref{lem:B-NCD}}
\label{app:proof-cb}

Let $B_1$, $B_2$, $\dots$, $B_n$ be $n$ arbitrary $B-$random variables. Consider orderings of $B_i$ that represent the sequence of events happening during episodes of the algorithm. For $n=3$, $(B_2, B_3, B_1)$ is an example ordering that corresponds to the situation where $B_2$ was initialised with a value either in the same, or previous episode compared to $B_3$; with a similar relationship for $B_3$ and $B_1$. Let $\Lambda_1$, $\Lambda_2$, $\dots$, $\Lambda_n$ be random variables that correspond to indices in any ordering. $\Lambda_i$ takes value $j$ if $B_j$ is at index $i$ in the ordering. 

We want to show the following results: 

\begin{equation} \label{NCD_one}
    \mathbb{P}\{\bigcap_i^n(B_i = 1)\} \leq \prod_i^n \mathbb{P}\{(B_i = 1)\};
\end{equation}

\begin{equation} \label{NCD_zero}
    \mathbb{P}\{\bigcap_i^n(B_i = 0)\} \leq \prod_i^n \mathbb{P}\{(B_i = 0)\}.
\end{equation}

We start by conditioning the LHS of equation \ref{NCD_one} on all the possible values of $\Lambda_i$'s.

\begin{align*}
    & \mathbb{P}\{\bigcap_i^n(B_i = 1)\} \\ =& \sum_{\Lambda_1 = \lambda_1}\sum_{\Lambda_2 = \lambda_2}\dots\sum_{\Lambda_n = \lambda_n}\mathbb{P}\{(B_1=1)\cap(B_2=1)\cap\dots\cap(B_n=1)\cap(\Lambda_1=\lambda_1)\cap(\Lambda_2=\lambda_2)\cap\dots\cap(\Lambda_n=\lambda_n)\} \\
    =& \sum_{\Lambda_1 = \lambda_1}\sum_{\Lambda_2 = \lambda_2}\dots\sum_{\Lambda_n = \lambda_n}\mathbb{P}\{(\Lambda_1=\lambda_1)\}\times\mathbb{P}\{(B_{\lambda_1}=1)|(\Lambda_1=\lambda_1)\}\times\mathbb{P}\{(\Lambda_2=\lambda_2)|(\Lambda_1=\lambda_1)\cap(B_{\lambda_1}=1)\}\\ & \times 
    \mathbb{P}\{(B_{\lambda_2}=1)|(\Lambda_1=\lambda_1)\cap(B_{\lambda_1}=1)\cap(\Lambda_2=\lambda_2)\}\times\dots\times\mathbb{P}\{(\Lambda_n=\lambda_n)|\dots\}\times\mathbb{P}\{(B_{\lambda_n}=1)|\dots\}.
\end{align*}

Consider a term in this product of the form $\mathbb{P}\{(B_{\lambda_m}=1)|(\Lambda_1=\lambda_1)\cap\dots\cap(\Lambda_m=\lambda_m)\cap(B_{\lambda_1}=1)\cap\dots\cap(B_{\lambda_{m-1}}=1)\}$, where $1 \leq m \leq n$. As the ordering is decided by $\Lambda_1\dots\Lambda_m$, there are two cases:
\begin{enumerate}
    \item $B_{\lambda_{m-1}}$ and $B_{\lambda_{m}}$ were initialised in the same episode, or
    \item $B_{\lambda_{m}}$ was initialised after $B_{\lambda_{m-1}}$.
\end{enumerate}

In the first case, $\mathbb{P}\{(B_{\lambda_{m}}=1)|(\Lambda_1=\lambda_1)\cap\dots\cap(\Lambda_m=\lambda_m)\cap(B_{\lambda_1}=1)\cap\dots\cap(B_{\lambda_{m-1}}=1)\} = 0$, and in the second case, $\mathbb{P}\{(B_{\lambda_{m}}=1)|(\Lambda_1=\lambda_1)\cap\dots\cap(\Lambda_m=\lambda_m)\cap(B_{\lambda_1}=1)\cap\dots\cap(B_{\lambda_{m-1}}=1)\} = \mathbb{P}\{(B_{\lambda_{m}}=1)$. Hence we have, $\mathbb{P}\{(B_{\lambda_{m}}=1)|(\Lambda_1=\lambda_1)\cap\dots\cap(\Lambda_m=\lambda_m)\cap(B_{\lambda_1}=1)\cap\dots\cap(B_{\lambda_{m-1}}=1)\} \leq \mathbb{P}\{(B_{\lambda_{m}}=1)$. Applying this substitution to all such terms, 

\begin{align*}
    & \mathbb{P}\{\bigcap_i^n(B_i = 1)\} \\ &\leq \prod_i^n \mathbb{P}\{(B_i = 1)\}
    \sum_{\Lambda_1 = \lambda_1}\sum_{\Lambda_2 = \lambda_2}\dots\sum_{\Lambda_n = \lambda_n}\mathbb{P}\{(\Lambda_1=\lambda_1)\}\times\mathbb{P}\{(\Lambda_2=\lambda_2)|(\Lambda_1=\lambda_1)\cap(B_{\lambda_1}=1)\}\times
    \dots\times\mathbb{P}\{(\Lambda_n=\lambda_n)|\dots\} \\
    &\leq \prod_i^n \mathbb{P}\{(B_i = 1)\}
    \sum_{\Lambda_1 = \lambda_1} \mathbb{P}\{(\Lambda_1=\lambda_1)\}\sum_{\Lambda_2 = \lambda_2} \mathbb{P}\{(\Lambda_2=\lambda_2)|(\Lambda_1=\lambda_1)\cap(B_{\lambda_1}=1)\} \sum_{\Lambda_3 = \lambda_3} 
    \dots \sum_{\Lambda_n = \lambda_n} \mathbb{P}\{(\Lambda_n=\lambda_n)|\dots\}.
\end{align*}
After rearranging the remaining terms inside the summations, each summation collapses to sum up to 1, leading to the result in  \eqref{NCD_one}. 


We now present the key steps in the proof of \eqref{NCD_zero}. Let $D_m$ be the set of $B-$random variables that are indexed lower than $m$ and were generated in the same episode as $B_{\Lambda_m}$. Note that $D_m$ may be empty. It follows that 
\begin{align}
&\mathbb{P}\{(B_{\lambda_{m}}=0)|(\Lambda_1=\lambda_1)\cap\dots\cap(\Lambda_m=\lambda_m)\cap(B_{\lambda_1}=0)\cap\dots\cap(B_{\lambda_{m-1}}=0) \cap D_m = d_m\}  \nonumber \\
\leq & 1-\dfrac{\mathbb{P}(B_{\lambda_{m}} = 1)}{1 - \sum_{j \in d_m} \mathbb{P}(B_{\lambda_{j}} = 1)} \leq 1-\mathbb{P}(B_{\lambda_{m}} = 1) = \mathbb{P}(B_{\lambda_{m}} = 0). \nonumber
\end{align}
By the law of total probability, 
\begin{align}
&\mathbb{P}\{(B_{\lambda_{m}}=0)|(\Lambda_1=\lambda_1)\cap\dots\cap(\Lambda_m=\lambda_m)\cap(B_{\lambda_1}=0)\cap\dots\cap(B_{\lambda_{m-1}}=0)\}\leq  1-\dfrac{\mathbb{P}(B_{\lambda_{m}} = 1)}{1 - \sum_{j \in d_m} \mathbb{P}(B_{\lambda_{j}} = 1)} \leq 1-\mathbb{P}(B_{\lambda_{m}} = 1) = \mathbb{P}(B_{\lambda_{m}} = 0).
\label{eq:conditional2}
\end{align}
Using \eqref{eq:conditional2}, the proof of \eqref{NCD_zero} follows along the same lines as the proof of \eqref{NCD_one}.


\section{Proof of Theorem~\ref{thm:ucb-t-regret}}
\label{app:proof-ucb}

The theorem states that there exists $c > 0$ such that for $T \geq 2$, the regret $R_{T}$ of the \textsc{Ucb-T} algorithm satisfies $$R_{T} \leq c \cdot \sum_{\sigma \in \Sigma, Y(\sigma) \not\subseteq \Pi_{\text{opt}}} \left( \frac{||\Sigma|\Delta^{\max}_{\sigma}}{(\Delta^{\min}_{\sigma})^{2}} \ln T + \ln |\Pi| + \ln |\Sigma| \right).$$ Consider episode $1 \geq t \leq T$ in a run over $T \geq 2$ episodes. On episode $t$, a policy $\pi \notin \Pi_{\text{opt}}$ is called \textit{needy} if $\beta(n^{t}(\pi), \delta_{U}(T)) \geq \frac{\Delta_{\pi}}{2}$. As in the proof for \textsc{Lucb-T}, a terminal state $\sigma \in \Sigma$ is called needy if it is the least-played state of some needy policy $\pi \in Y(\sigma)$. 


Again, we observe that \textit{non-needy} policies have sufficiently small widths $\beta$. Hence, on episode $t$, a \textit{non-needy} policy $\pi \notin \Pi_{\text{opt}}$ can be played (due its having the highest UCB) only if (1) $\text{ucb}^{t}(\pi^{\star}, \delta_{U}(t)) < V(\pi^{\star})$ or (2) $\text{lcb}^{t}(\pi, \delta_{U}(t)) > V(\pi)$. Applying Theorem~\ref{thm:confidencebounds} along with a union bound, we upper-bound the probability of any non-needy policy being played in episode $t$ by $|\Pi||\Sigma| t^{|\Sigma|} (2 \delta_{U}(t)) = \frac{2}{t^{3}}.$ The expected number of plays of non-needy policies up to episode $T$ is at most $\sum_{t = 1}^{T} \frac{2}{t^{3}} \leq 2.$ Each episode trivially has regret at most $1$. 


If a \textit{needy} policy is played on an episode, then it means at least one needy terminal state $\sigma$ is played on that episode. The number of episodes for which $\sigma$ is needy and played is at most $\lceil \frac{2}{(\Delta^{\text{min}}_{\sigma})^{2}} \ln \frac{1}{\delta_{U}(T)} \rceil$, since beyond that its confidence width cannot exceed $\frac{\Delta_{\pi}}{2}$ for any $\pi \in Y(\sigma) \setminus \Pi_{\text{opt}}$. The contribution of $\sigma$ to the regret on this episode is at most $\Delta^{\max}_{\sigma}$. Hence, in aggregate, the cumulative regret from episodes on which needy policies are played is at most $$\sum_{\sigma \in \Sigma, Y(\sigma) \not\subseteq \Pi_{\text{opt}}} \frac{2 \Delta^{\max}_{\sigma}}{(\Delta^{\min}_{\sigma})^{2}} \ln \frac{1}{\delta_{U}(T)}.$$ Summing the cumulative regret from the plays of needy and non-needy policies yields the claimed upper bound.

\section{Description of the Games in the Test Suite} \label{app:gamedesc}

Kuhn Poker \cite{kuhn1953extensive} and Leduc Poker \cite{BayesBluffOppModelling} are poker variants that serve as popular benchmarks for learning algorithms in AI research. Kuhn Poker has three cards (King, Queen and Jack), while Leduc has six (2 Kings, 2 Queens and 2 Jacks). We also show experiments for a variant of Kuhn Poker with 5 cards (Ten, Jack, Queen, King and Ace), which has is identical to the three card variant in every other way. In both games, players are dealt a single card and have fixed betting amounts. Kuhn Poker consists of a single betting round after which there is a showdown, whereas in Leduc Poker, a community card is revealed after the first round, followed by another round and showdown. Leduc Poker also allows for a single raise in each betting round, with fixed raise amounts of 2 and 4. The utilities at terminal states depend on the size of the pot won - with utilities ranging between $[-2,2]$ for Kuhn Poker and $[-13,13]$ for Leduc Poker.

RBT \cite{rbt} is an imperfect information version of Tic Tac Toe. Players cannot see each other's moves, instead gaining information about the true state through a reconnaissance move (also called a sense move), where they select a $2 \times 2$ region of the Tic Tac Toe board for observation. Each move is preceded by a sense move, after which players must select one of the empty squares on the board to play an `$\mathtt{x}$' or an `$\mathtt{o}$'. Figure \ref{fig:RBT_states} shows possible true board configurations given the information available to a player. Since players are not aware of the true state, they might choose a square that already contains an opponent move. By convention, this results in an immediate loss. It is important to note that unlike in poker, all actions are private in RBT. This leads to a different number of terminal states for each player, whereas in poker both players share the same terminal states. Utilities belong to the set $\{-1,0,1\}$, corresponding to a loss, a draw and a win, respectively. Figure \ref{fig:RBT_history} shows an example of the first few moves of a game of RBT. Table \ref{tab:games_mdp} coalesces information about the parameters of each game.

\begin{figure}[ht]
    \centering
    \subfloat[]{%
        \includegraphics[width=0.30\linewidth]{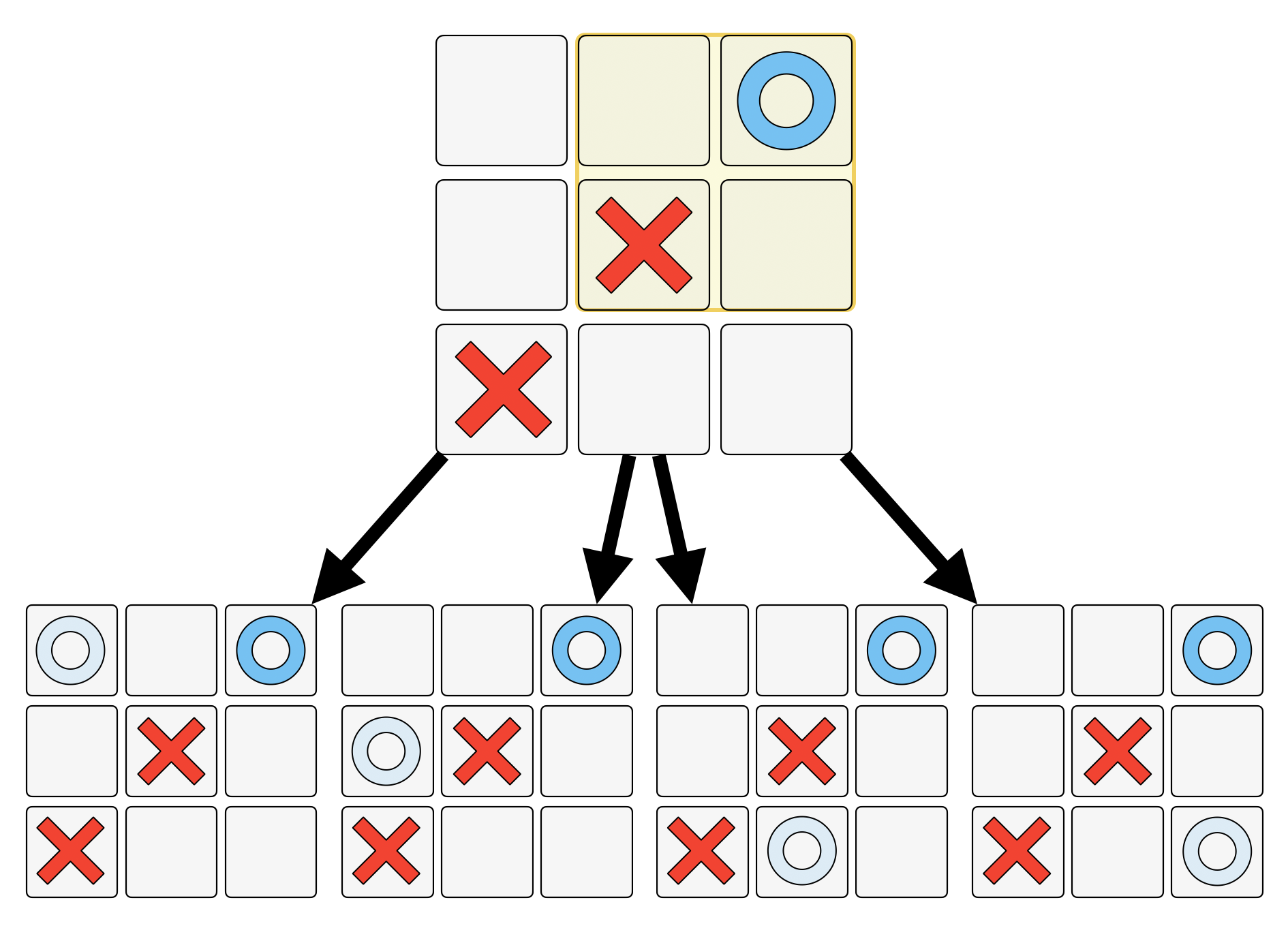}%
        \label{fig:RBT_states}%
    }

    \subfloat[]{%
        \includegraphics[width=0.66\linewidth]{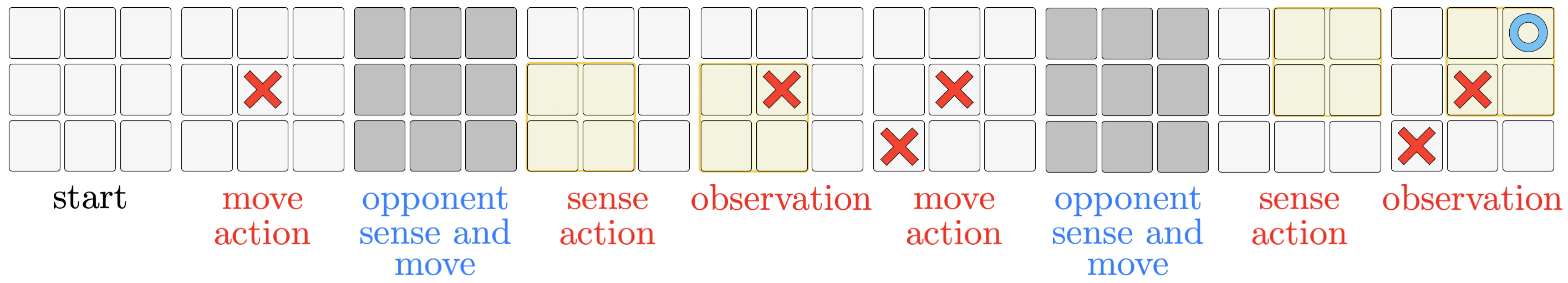}%
        \label{fig:RBT_history}%
    }
    \caption{(a) A visualisation of the uncertainty in RBT after playing a sense move. The sense region (in yellow) reveals information about one of the opponent moves ($\mathtt{o}$), but the other is unknown to player $\mathtt{x}$ \textemdash it could be in any of the four locations shown below. (b) A possible history of actions and observations in RBT.}
    \label{fig:combined_RBT_figures}
    \Description{This figure shows examples of RBT games and histories. }
\end{figure}

\begin{table}[ht]
    \caption{Details of the MDP formulation of the games in the test suite.}
    \Description{Details of the MDP formulation of the games in the test suite.}
    \label{tab:games_mdp}
    \begin{center}
        \begin{tabular}{lccccc}
        \toprule
        Game & $\mathcal{S}$ (\# of information sets) & $\mathcal{A}$ & r & $\Sigma$ (\# of terminal states) & $\mathcal{H}$ \\
        \midrule
        3-card Kuhn Poker & 6 per player & 4 (check, bid, call, fold) & $[-2, 2]$ & 15 & 2\\
        5-card Kuhn Poker & 10 per player & 4 (check, bid, call, fold) & $[-2, 2]$ & 25 & 2\\
        Leduc Poker & 144 per player & 5 (check, bid, call, raise, fold) & $[-13, 13]$ & 417 & 4\\
        RBT & \shortstack{$10^7$ for player $\mathtt{x}$ \\ $2\times10^7$ for player $\mathtt{o}$} & 9 move actions, 4 sense actions & $[-1, 1]$ & \shortstack{$2 \times 10^7$ for player $\mathtt{x}$ \\ $5\times10^7$ for player $\mathtt{o}$} & 9\\
        \bottomrule
        \end{tabular}
    \end{center}
\end{table}

\section{More Kuhn Poker Experiments}
\label{app:kuhnexp}

Table \ref{tab:o_kuhn_stopping_times} shows a comparison of PAC stopping times for variants of \textsc{Lucb} on 3-card Kuhn Poker and 5-card Kuhn Poker for player \texttt{o}. We see similar trends as in Table~\ref{tab:x_kuhn_stopping_times}. 

\begin{table*}[h]
\caption{PAC stopping times for player \texttt{o} in 3-card compared to 5-card Kuhn Poker (averaged over 10 experiments).}
\Description{PAC stopping times for player \texttt{o} in 3-card compared to 5-card Kuhn Poker (averaged over 10 experiments).}
\centering
\begin{tabular}{@{}lcccc@{}}
\toprule
\multirow{2}{*}{\textbf{Algorithm}} & \multicolumn{4}{c}{\textbf{Experiment Parameters}} \\
\cmidrule(l){2-5}
& \multicolumn{2}{c}{\textbf{3-card Kuhn Poker}} & \multicolumn{2}{c}{\textbf{5-card Kuhn Poker}} \\
\cmidrule(l){2-3} \cmidrule(l){4-5}
& \multicolumn{1}{c}{player \texttt{o}, $\epsilon = 0.05, \delta = 0.05$} & \multicolumn{1}{c}{player \texttt{o}, $\epsilon = 0.1, \delta = 0.1$} & \multicolumn{1}{c}{player \texttt{o}, $\epsilon = 0.05, \delta = 0.05$} & \multicolumn{1}{c}{player \texttt{o}, $\epsilon = 0.1, \delta = 0.1$} \\ \midrule
\textsc{Lucb} & $\mathbf{0.163 \times 10^6 \pm 4.3\%}$ & $\mathbf{0.057 \times 10^6 \pm 4.0\%}$ & $\mathbf{3.154 \times 10^6 \pm 1.5\%}$ & $1.291 \times 10^6 \pm 0.9\%$ \\
\textsc{Lucb-T} & $1.896 \times 10^6 \pm 2.4\%$ & $0.496 \times 10^6 \pm 4.1\%$ & $3.475 \times 10^6 \pm 4.4\%$ & $\mathbf{0.861 \times 10^6 \pm 4.5\%}$ \\
\textsc{Lucb-T-Uniform} & $1.671 \times 10^6 \pm 1.5\%$ & $0.401 \times 10^6 \pm 4.1\%$ & $4.150 \times 10^6 \pm 7.1\%$ & $1.045 \times 10^6 \pm 6.4\%$ \\ \bottomrule
\end{tabular}
\label{tab:o_kuhn_stopping_times}
\end{table*}

\section{Instance-dependent PAC Bounds}
\label{app:instance-dependen-pac-bounds}

In this section, we review the recent developments in the instance-dependent PAC regime for MDPs. For the sample complexity bounds, the notation $\mathcal{O}$ hides absolute constants, whereas $\Tilde{\mathcal{O}}$ suppresses poly logarithmic terms along with absolute constants.

\begin{itemize}
    \item \textbf{\textsc{Mdp-gape}}: One of the early works to provide an instance-dependent PAC bound was \cite{MDPGAPE}. The authors propose \textsc{Mdp-gape}, an adaptation of the UGapE bandit algorithm \cite{DBLP:conf/nips/GabillonGL12} for episodic MDPs. In this approach, the sampling rule selects an action using the UGap-E algorithm at the starting state, while at the other states, it is optimistic w.r.t. the upper confidence bounds on the values of the state-action pairs. The sample complexity of \textsc{Mdp-gape} was shown to be $$\mathcal{O}\left(\sum_{a_1}\frac{\mathcal{H}(BK)^{\mathcal{H}-1}}{(\Delta_1(s_1, a_1) \vee \epsilon )^2} \left ( \log \nicefrac{1}{\delta} + B\mathcal{H}\log(BK)  \right ) \right),$$ where $\mathcal{H}$ is the horizon, $K$ denotes the maximum number of actions available at any state, $B$ denotes the branching factor which can been seen as the maximum number of next states that can be reached by taking any action at any state, and $\Delta_1(s_1, a_1)=V^{\pi^\star}(s_1) - Q^{\star}(s_1, a_1)$ denotes the sub-optimality of action $a_1$ at state $s_1$ and stage $1$. While this bound features {\em value gaps} $\Delta(s, a)$ of the actions, \textsc{Mdp-gape} is shown to only identify a PAC action at the starting state of the MDP.
    \item \textbf{\textsc{Moca}}: For PAC policy identification, \cite{DBLP:conf/colt/WagenmakerSJ22} proposed \textsc{Moca} (monte carlo action elimination) with a sample complexity bound of $$\Tilde{\mathcal O}\left(\left (\sum_{h=1}^{\mathcal{H}}\inf_{\pi}\max_{s,a}\min \left \{\frac{1}{p^\pi_h(s,a)\ \Delta_h(s,a)^2}, \frac{W_h(s)^2}{p^\pi_h(s,a)\ \epsilon^2} \right \} +  \frac{\mathcal{H}^2 |OPT(\epsilon)|}{\epsilon^2}\right )\log (\nicefrac{1}{\delta})\right),$$ that depends on {\em maximum reachability} $W_h(s)$ of a state at step $h$, the {\em value gaps} $\Delta_h(s,a)$ of the state-action pairs at step $h$, and the number of {\em near-optimal} state-action pairs $|OPT(\epsilon)|$ in the MDP\footnote{$p^\pi_h(s,a)$ in the sample complexity bound denotes the probability of reaching $(s,a)$ at stage $h$ under policy $\pi$.}. The intuition behind the algorithm is to learn to aggressively explore the MDP in order to eliminate the sub-optimal actions at each state (first term in the bound). 
    Using a few additional samples for the remaining near-optimal state-action pairs (second term in the bound), a PAC policy is identified. \cite{DBLP:conf/colt/WagenmakerSJ22} also proved that low regret algorithms\textemdash in particular, optimistic algorithms\footnote{Recall that these are algorithms in which the sampling rule selects actions at each state greedily with respect to some UCB quantity.}\textemdash fail to achieve the instance-optimal sample complexity bound on a class of MDPs, and that in general a more aggressive exploration is necessary instead on relying on optimistic exploration.
    \item \textbf{\textsc{Bpi-Ucrl}}: While optimistic algorithms cannot achieve the instance-optimal sample complexity in episodic MDPs, \cite{DBLP:conf/alt/TirinzoniMK23} provide new insights in to the instance-specific sample complexity of an optimistic algorithm \textsc{Bpi-Ucrl}, which has been previously analyzed in \cite{pmlr-v132-kaufmann21a}. The sample complexity, $$\Tilde{\mathcal O}\left ( \left ( \sum_{h \in \mathcal{H}}\sum_{s\in \mathcal{S}}\sum_{a \in \mathcal{A}}\frac{\mathcal{H}^4}{p^{\text{min}}_h(s,a)\ \max{\{\Tilde{\Delta}_h(s,a)^2}, \epsilon\}^2}\right )\log \nicefrac{1}{\delta}\right ),$$ features $p^{\text{min}}_h(s,a)$ which is the minimum probability of reaching $(s,a)$ at stage $h$ across all deterministic policies $\pi \in \Pi$, and a new notion of sub-optimality gap for state-action pairs called the {\em conditional return gap}, $$\Tilde{\Delta}_h(s,a) = \min_{\pi \in \Pi: p^\pi_h(s,a)>0} \max_{\ell \in [\mathcal{H}]} \max_{s' \in \mathcal{S}: p^{\pi}_{\ell}(s')>0} V^{\star}_{\ell}(s') - V^{\pi}_{\ell}(s').$$ Intuitively, these gaps capture the following idea. In order to learn $(s,a, h)$ is suboptimal, all polices visiting $(s,a,h)$ must be identified as suboptimal. For any such policy $\pi$, the gap definition states that the hardness of determining its sub-optimality is proportional to the {\em maximum value gap} at a state reachable by $\pi$. This reflects the fact that the easiest place to identify the sub-optimality of a policy is where the gap is the largest. While $\Tilde{\Delta}_h(s,a)$ has been shown to be larger\textemdash therefore, better\textemdash than the value gaps of \cite{DBLP:conf/colt/WagenmakerSJ22}, $p^{\text{min}}_h(s,a)$ can be arbitrarily small.
    \item  \textbf{\textsc{Pedel}}: Since the upper bounds across different algorithms are incomparable, a different line of work has focused on establishing lower bounds on the number of samples required for PAC policy identification \cite{DBLP:journals/corr/abs-2311-05638lowerbound}. In particular, it was shown that the best $\log\nicefrac{1}{\delta}$ dependency that any algorithm can achieve in order to identify a PAC-optimal policy is $$\left (2 \min_{\pi^{\epsilon} \in \Pi^{\epsilon}} \min_{\rho \in \Omega} \max_{\pi \in \Pi}\sum_{s,a,h}\frac{(p_h^{\pi}(s,a) - p_h^{\pi^{\epsilon}}(s,a))^2}{\rho_h(s,a)(\Delta(\pi)-\Delta(\pi^{\epsilon}) + \epsilon)^2}\right ) \log \nicefrac{1}{\delta}.$$ \cite{DBLP:journals/corr/abs-2311-05638lowerbound} also showed that the \textsc{Pedel} algorithm \cite{DBLP:conf/nips/WagenmakerJ22} closely matches this lower bound. However, \textsc{Pedel} is known to be computationally intractable, and the problem of designing an efficient algorithm with matching lower bounds remains open.
\end{itemize}

Our results for T-MDPs are distinct from the works in the literature because our bounds avoid the dependence on the state-visitation probability ($q(\sigma)$ in our work)---which can be arbitrarily small. 
The only other work we are aware of that avoids this dependence is that of \cite{NearInstanceOptimalPAC}. However, the authors focus on {\em deterministic} MDPs, where the state-transitions are deterministic and the rewards are stochastic.

\section{Code and Hyperparameter Settings}
\label{app:code-hyperparameter}

We have selected $\delta = 0.05$ for all experiments in the PAC paradigm. Horizon dependence in \textsc{Bpi-Ucrl} and \textsc{Mdp-GapE} removed as rewards are only terminal in poker. All experiments use a discount factor of 1. Kuhn and Leduc Poker can be run locally. RBT required a server with 96 cores and a total of 120 GiB RAM, where each core was an Intel(R) Xeon(R) Gold 6342 CPU @ 2.80GHz. Most of the steps cannot be parallelized, but steps such as expected utility and optimal best response computation can. RBT requires between 50-100 GB of RAM to run. 

\begin{table}[h]
    \centering
    \caption{Hyperparameters and implementation details of algorithms in test suite.}
    \Description{Hyperparameters and implementation details of algorithms in test suite.}
    \label{tab:algorithms}
    \begin{tabular}{p{2cm} ccc ccc p{6cm}}
        \toprule
        Algorithm & \multicolumn{6}{c}{Hyperparameters} & Implementation Details and Comments \\
        \cmidrule(lr){2-7}
        & \multicolumn{2}{c}{Kuhn Poker} & \multicolumn{2}{c}{Leduc Poker} & \multicolumn{2}{c}{RBT} & \\
        \cmidrule(lr){2-3}\cmidrule(lr){4-5}\cmidrule(lr){6-7}
        & $\mathtt{x}$ & $\mathtt{o}$ & $\mathtt{x}$ & $\mathtt{o}$ & $\mathtt{x}$ & $\mathtt{o}$ & \\
        \midrule
        \textsc{Bpi-Ucrl} & - & - & - & - & - & - & We use the following values for the exploration term $\beta(\delta, n_t)$ in our code: $\beta_{\text{cnt}} = \ln\left(\frac{1}{\delta}\right)$, $\beta_r = \frac{1}{2}\left(\beta_{\text{cnt}} + \ln(e(1 + n_t)\right)$, and $\beta = \sqrt{\frac{\beta_r}{n_t}}$.\\
        \textsc{Lucb-T} &$C=0.1$ &$C=1.0$ & $C=1.0$& $C=1.0$&$C=1.0$ &$C=1.0$ & We use a simplified confidence bound expression in our code: $\beta(m, \delta) = \sqrt{\frac{C}{m}\ln{\frac{t}{\delta}}}$.\\
        \textsc{Mccfr} &$\epsilon = 0.1$ &$\epsilon = 0.1$ &$\epsilon = 0.1$ & $\epsilon = 0.1$& $\epsilon = 0.2$& $\epsilon = 0.2$& Implementation as described in literature. \\
        \textsc{MDP-GapE} & - & -  & - & - & - & - & We use the following values for the exploration terms $\beta_r(\delta, n_t)$ and $\beta_p(\delta, n_t)$ in our code: $\beta_{\text{cnt}} = \ln\left(\frac{3 \cdot (6B)^H}{\delta}\right)$, $\beta_r = \beta_{\text{cnt}} + \ln(1 + n_t) + 1$, and $\beta_p = \beta_{\text{cnt}} + (B-1)\left(1 + \ln\left(1 + \frac{n_t}{B-1}\right)\right)$.\\
        \textsc{Opf} & $K=0.5$& $K=0.5$& $K=5.0$& $K=5.0$& $K=5.0$& $K=5.0$& Implementation as described in literature. \\
        \textsc{Ucb-T} & $C=0.1$& $C=0.1$&$C=1.0$ & $C=1.0$& $C=1.0$&$C=1.0$ & We use a simplified confidence bound expression in our code: $\beta(m, \delta) = \sqrt{\frac{C}{m}\ln{\frac{t}{\delta}}}$. \\
        \textsc{Uct} & $C=0.5$&$C=0.5$ & $C=5.0$& $C=5.0$& $C=2.0$& $C=2.0$& Implementation as described in literature.\\
        \bottomrule
    \end{tabular}
\end{table}

\section{The PU Procedure}
\label{app:pu-procedure}

In this section, we provide the detailed algorithm and correctness argument for the PU procedure described in Section~\ref{subsec:algorithms-lucb-t}. 
The procedure computes the policy $\pi^{t}_{U} \eqdef \argmax_{\pi \in \Pi} \text{ucb}^{t}(\pi, \delta_{U}(t))$ efficiently.

\subsection{Notation and Definitions}

For a state $s \in \mathcal{S}$ and action $a \in \mathcal{A}$, let $\mathcal{C}(s, a)$ denote the set of children states reachable from $s$ by taking action $a$. That is, $\mathcal{C}(s, a) = \{s' \in \mathcal{S} \cup \Sigma : p(s, a, s') > 0\}$.
We assume that the empirically best policy $\pi^t_1$ and its empirical values $\widehat{V}^{\pi_1^t}(\cdot)$ have been precomputed for all states.  
Recall that the Upper Confidence Bound (UCB) of a policy $\pi$ is the sum of its empirical value $\widehat{V}^{\pi}$ and the confidence width $\beta(n^t(\pi), \delta_{U}(t))$. 
Observe that the problem of computing $\pi^t_U$ is a maximisation problem over these two terms of empirical value and confidence width. 
Let $\sigma^t_{\min}(\pi) \eqdef \argmin_{\sigma \in X(\pi)} n^t(\sigma)$ be the terminal state that determines the width for policy $\pi$.
If $\sigma^t_{\min}(\pi^t_U)$ is somehow known, we can compute $\pi^t_U$ by maximising over the empirical value of all policies that reach $\sigma^t_{\min}(\pi^t_U)$. This is easily done by modifying $\pi^t_1$ to select actions that reach $\sigma^t_{\min}(\pi^t_U)$, while leaving all other actions unchanged. 
Below we specify a recursive procedure to find $\sigma^t_{\min}(\pi^t_U)$. 

First, let us extend the notation for previously defined quantities at the root of the T-MDP to a subtree rooted at state $s$. 
Let $X(\pi, s) = \{\sigma \in X(\pi): \sigma$ is reachable from $s$ by following policy $\pi\}$. 
Let $n^t(\pi, s) \eqdef \min_{\sigma \in X(\pi, s)} n^t(\sigma)$ be the minimum play count of policy $\pi$ in the subtree rooted at $s$. 
Next, consider a state-action pair $(s, a)$ with children $\mathcal{C}(s, a) = \{s'_1, \dots, s'_m\}$. Assume that we have already computed the policy with the highest UCB for each subtree rooted at $s'_j$ for all $j \in \{1, \dots, m\}$, in isolation to the rest of the game tree. 
Let $\pi^{t, j}_U$ denote this `locally' optimistic policy for the subtree rooted at $s'_j$, defined only for states in the subtree. 
Then we can compute the action that maximises the UCB for state $s$ by maximising over all $m$ possible assignments of policies to the children, where in each assignment, one child $s'_k$ follows the locally optimistic policy $\pi^{t, k}_U$ and all other children follow the empirically best policy $\pi^t_1$. 
Essentially, the argument is that since $\sigma^t_{\min}(\pi^t_U)$ belongs to exactly one of the subtrees rooted at $s'_j$, actions in the other subtrees can be chosen to maximise the empirical value of the policy. 
Thus we have,   
$$ \pi^t_U(s) = \argmax_{a \in \mathcal{A}} \max_{s_k \in \mathcal{C}(s,a)} \left[ \sum_{s_j \in \mathcal{C}(s,a), j \ne k} \widehat{p}(s,a,s_j)\widehat{V}^{\pi_1}(s_j) + \widehat{p}(s,a,s_k)\widehat{V}^{\pi^{t, k}_U}(s_k) + \beta(n^t(\pi^{t, k}_U, s_k), \delta_{U}(t)) \right], $$
and
$$ \pi^t_U(s') = \begin{cases} \pi^t_1(s') & \text{if } s' \in \text{subtree rooted at } s'_j \text{ for all } j \ne k, \\
\pi^{t, k}_U(s') & \text{if } s' \in \text{subtree rooted at } s'_k. \end{cases} $$


\subsection{Algorithm}

The algorithm \textsc{FindMaxUCB}$(s)$ returns a tuple $(U, \sigma, n)$, where $U$ is the maximal UCB value for the sub-tree at $s$, $\sigma$ is the terminal state minimising the count reachable from $s$ by following the UCB policy $\pi^t_U$, and $n$ is that minimum count.

\begin{algorithm}[H]
\caption{PU Procedure: \textsc{FindMaxUCB}$(s)$}
\label{alg:pu_detailed}
\begin{algorithmic}[1]
    \STATE \textbf{Input:} State $s$.
    \STATE \textbf{Global:} Precomputed $\widehat{V}^{\pi_1}(\cdot)$ and $\widehat{p}(\cdot)$. 
    \IF{$s \in \Sigma$}
        \RETURN $(\widehat{V}^{\pi_1}(s) + \beta(n^t(s), \delta_{U}(t)), s, n^t(s))$.
    \ENDIF
    \STATE $U^* \gets -\infty$.
    \FOR{each action $a \in \mathcal{A}$}
        \STATE Let $\mathcal{C}(s, a) = \{s'_1, \dots, s'_m\}$.
        \FOR{$k = 1$ to $m$}
            \STATE $(U_k, \sigma_k, n_k) \gets \textsc{FindMaxUCB}(s'_k)$
            \STATE $U_{curr} \gets \widehat{V}^{\pi_1}(s) + \widehat{p}(s, a, s'_k)(U_k - \beta(n_k, \delta_U(t)) - \widehat{V}^{\pi_1}(s'_k)) + \beta(n_k, \delta_U(t))$.
            \IF{$U_{curr} > U^*$}
                \STATE $U^* \gets U_{curr}$.
                \STATE $n^* \gets n_k$.
                \STATE $\sigma^* \gets \sigma_j$.
            \ENDIF
        \ENDFOR
    \ENDFOR
    \RETURN $(U^*, \sigma^*, n^*)$.
\end{algorithmic}
\end{algorithm}


\end{document}